\documentclass[ejs]{imsart}

\usepackage[numbers, sort]{natbib}
\usepackage{graphicx}
\usepackage{amsmath}
\usepackage{amsfonts}
\usepackage{mathtools}
\usepackage[mathscr]{euscript}
\usepackage[export]{adjustbox}
\usepackage[caption=false]{subfig}
\usepackage{algorithm}
\usepackage{algpseudocode}
\usepackage{amssymb}
\usepackage{rotating}
\usepackage[dvipsnames]{xcolor}
\usepackage{multicol}
\usepackage{multirow}
\usepackage{hyperref}
\usepackage{enumitem}

\pubyear{2024}
\volume{18 | No. 2}

\startlocaldefs
\def\I{{\mathbb I}}
\def\P{{\mathbb P}}
\def\pr{\mathbb{P}}

\def\Or{\mathbb{OR}}

\def\X{{\mathbf{X}}}
\def\x{{\mathbf{x}}}

\def\E{{\mathbb E}}

\def\D{{\mathcal D}}
\def\T{{\mathcal T}}

\def\ACORE{\texttt{ACORE }}
\def\BFF{\texttt{BFF }}

\let\hat\widehat

\newtheorem{prop}{Proposition}
\newtheorem{thm}{Theorem}
\newtheorem{Assumption}{Assumption}
\newtheorem{Lemma}{Lemma}
\newtheorem{Corollary}{Corollary}

\newcommand{\BlackBox}{\rule{1.5ex}{1.5ex}}  
\ifdefined\proof
    \renewenvironment{proof}{\par\noindent{\bf Proof\ }}{\hfill\BlackBox\\[2mm]}
\else
    \newenvironment{proof}{\par\noindent{\bf Proof\ }}{\hfill\BlackBox\\[2mm]}
\fi

\newcommand{\codecomment}[1]{\textbf{\color{black}// #1}}

\definecolor{awesome}{rgb}{1.0, 0.13, 0.32}
\definecolor{safetyorange}{rgb}{1.0, 0.4, 0.0}
\definecolor{vermilion}{rgb}{0.89, 0.26, 0.2}

\newcommand{\revJMLR}[1]{{\color{black} #1}}

\newcommand{\revEJS}[1]{{\color{black} #1}}

\newcommand{\addEJS}[1]{{\color{black} #1}}

\newcommand{\remove}[1]{\textbf{\color{lightgray} [#1]}}

\newcommand{\addTwo}[1]{{\color{black} #1}}

\endlocaldefs

\begin{document}

\begin{frontmatter}

\title{
Likelihood-Free Frequentist Inference:
\addEJS{Bridging Classical Statistics and \\ Machine Learning for Reliable\\ Simulator-Based Inference}\support{This work was supported in part by NSF DMS-2053804, NSF PHY-2020295, and the C3.ai Digital Transformation Institute. RI is grateful for the financial support of CNPq (422705/2021-7 and 305065/2023-8) and FAPESP (2019/11321-9 and 2023/07068-1).}
}
\runtitle{Likelihood-Free Frequentist Inference}

\author{
\fnms{Niccol\`o} \snm{Dalmasso}\textsuperscript{1,}\thanksref{t1}\ead[label=e1]{niccolo.dalmasso@gmail.com},
\fnms{Luca} \snm{Masserano}\textsuperscript{2,}\thanksref{t1}\ead[label=e2]{masserano.luca@gmail.com},\\
\fnms{David} \snm{Zhao}\textsuperscript{2}\ead[label=e3]{dzhaoism@gmail.com},
\fnms{Rafael} \snm{Izbicki}\textsuperscript{3}\ead[label=e4]{rafaelizbicki@gmail.com},
\fnms{Ann} \snm{B. Lee}\textsuperscript{2}\corref{}\ead[label=e5]{annlee@andrew.cmu.edu}
}

\address{\textsuperscript{1}Department of Statistics and Data Science, Carnegie Mellon University\\ \printead{e1}}
\address{\textsuperscript{2}Department of Statistics and Data Science, Machine Learning Department, \\ Carnegie Mellon University \\ \printead{e2}, \printead{e3} \\ \printead{e5}}
\address{\textsuperscript{3}Department of Statistics, Federal University of S\~{a}o Carlos \\ \printead{e4}}

\thankstext{t1}{Equal Contribution} 

\runauthor{Dalmasso, Masserano, Zhao, Izbicki, and Lee}

\begin{abstract}
Many areas of science rely on simulators that implicitly encode intractable likelihood functions of complex systems. Classical statistical methods are poorly suited for these so-called likelihood-free inference (LFI) settings, especially outside asymptotic and low-dimensional regimes. At the same time, popular LFI methods --- such as Approximate Bayesian Computation or more recent machine learning techniques --- \addTwo{do not necessarily lead to valid scientific inference because they do not guarantee confidence sets with nominal coverage in general settings.} In addition, LFI currently lacks practical diagnostic tools to check the actual coverage of computed confidence sets across the entire parameter space. In this work, we propose a modular inference framework that bridges classical statistics and modern machine learning to provide (i) a practical approach for constructing confidence sets with \addTwo{near finite-sample validity at} any value of the unknown parameters, and (ii) interpretable diagnostics for estimating empirical coverage across the entire parameter space. We refer to this framework as {\em likelihood-free frequentist inference} (LF2I). Any method that defines a test statistic can leverage LF2I to create  valid confidence sets and diagnostics without costly Monte Carlo \addTwo{or bootstrap} samples at fixed parameter settings. We study two likelihood-based test statistics (\texttt{ACORE} and \texttt{BFF}) and demonstrate their performance on high-dimensional complex data. Code is available at \href{https://github.com/lee-group-cmu/lf2i}{https://github.com/lee-group-cmu/lf2i}.
\end{abstract}





\end{frontmatter}


\section{Introduction}
\label{sec:intro}
Hypothesis testing and uncertainty quantification are the hallmarks of scientific inference. Methods that achieve good statistical performance (e.g., high power) often rely on being able to explicitly evaluate a likelihood function, which relates parameters of the data-generating process to observed data.  However, in many areas of science and engineering, complex phenomena are modeled by forward simulators that {\em implicitly} define a likelihood function. For example,\footnote{\noindent{\bf Notation.} 
Let $F_{\theta}$  
represent the stochastic forward \revJMLR{model} for a sample point $\X \in \mathcal{X}$ at parameter $\theta \in \Theta$. \revJMLR{We refer to $F_{\theta}$ as a ``simulator'', as the assumption is that we can sample data from the model.} We denote  i.i.d ``observable'' data from  $F_{\theta}$ by $\D=\left\{ \X_1,\ldots,\X_n \right\}$, and the actually observed or measured data by $D=\left\{ \x_1^{\text{obs}},\ldots,\x_n^{\text{obs}} \right\}$. The likelihood function  
\revEJS{is defined as $\mathcal{L}(D;\theta)=\prod_{i=1}^{n} p(\x_i^{\text{obs}}|\theta)$, where $p(\cdot|\theta)$  
is the density of $F_\theta$ with respect to a fixed dominating measure $\nu$\revEJS{, which could be the Lebesgue measure.}}}
given input parameters \revJMLR{$\theta$ from some parameter space $\Theta$}, 
a stochastic model \revJMLR{$F_\theta$} may encode the interaction of atoms or elementary particles\revEJS{, or the transport of radiation through the atmosphere or through matter in the Universe by combining deterministic dynamics with random fluctuations and measurement errors,} to produce synthetic data $\X$.
\\\\
Simulation-based inference \addEJS{with an intractable likelihood} is commonly referred to as {\em likelihood-free inference} (LFI). The most well-known approach to LFI is Approximate Bayesian Computation (ABC; see \cite{beaumont2002approximate, marin2012approximate, sisson2018handbook} for a review). These methods use simulations sufficiently close to the observed data   $D=\left\{ \x_1^{\text{obs}},\ldots,\x_n^{\text{obs}} \right\}$ to infer the underlying parameters, or more precisely, the posterior distribution $p(\theta|D)$. Recently, the arsenal of LFI methods has been expanded with new machine learning algorithms (such as neural density estimators) that instead use the output from simulators as training data. The objective here is to learn a ``surrogate model'' or {\em approximation} of the likelihood $p(D|\theta)$ or posterior \revJMLR{$p(\theta|D)$}. The surrogate model, rather than the simulations themselves, is then used for inference. \addEJS{Machine-learning (ML) based methods have revolutionized LFI in terms of the complexity and dimensionality of the problems that can be tackled (see \cite{Cranmer2020Review} for a recent review). Nevertheless, neither ABC nor ML-based LFI approaches guarantee confidence sets with frequentist coverage,} which are crucial to ensure reliability of downstream scientific conclusions. Suppose that we have a high-fidelity simulator $F_\theta$, which implicitly encodes the likelihood, and that we observe data $\D$ of finite sample size $n$. \addTwo{We address two open challenges in LFI:}\\

\vspace{-0.2cm}\noindent \addTwo{\textbf{i)} The first challenge is} finding practical procedures for constructing a $(1-\alpha)$ confidence set $R(\D)$ with nominal coverage\footnote{
 \addEJS{We  use the notation $\P_{\D|\theta}(\cdot)$ to emphasize the fact that $\D$ is random, but $\theta$ is fixed. }}
\begin{equation}\label{eq:cond_coverage}
\pr_{\mathcal{D}|\theta} \left( \theta \in R(\D)\right) =  1-\alpha , 
\end{equation} where $\alpha \in (0,1)$,
{\em regardless of} the  true value of the unknown parameter $\theta \in \Theta$ and of the number of observations $n$. Monte Carlo and bootstrap procedures are computationally infeasible for continuous parameter spaces $\Theta$, and large-sample theory does not apply when, e.g., $n=1$. \addTwo{The latter $n=1$ scenario is very common in, e.g., large astronomical surveys where each object (e.g., galaxy or star) has a different parameter value $\theta$ and may only be measured once.}\\

\vspace{-0.2cm} \noindent \addTwo{\textbf{ii)} The second challenge is} finding practical and interpretable procedures to check that the empirical coverage of the constructed sets $R(\D)$ is indeed close to (and no smaller than) $1-\alpha$ for {\em any} $\theta \in \Theta$ (again, without resorting to costly Monte Carlo simulations at fixed parameter settings on a fine grid in parameter space $\Theta$ \citep[Section 13]{cousins2018lectures}). \addTwo{Local validity across the entire parameter space is essential for reliable scientific inference because the scientist does not actually know what the true value of $\theta$ is for the object of interest.}

\paragraph{Novelty and significance} \addTwo{In this paper, we introduce a fully modular statistical framework that addresses both problems above. We refer to the general approach as {\em likelihood-free frequentist inference} (LF2I)\footnote{Code is available as a Python package at \href{https://github.com/lee-group-cmu/lf2i}{https://github.com/lee-group-cmu/lf2i}.}. LF2I is fully nonparametric and targets modern scientific applications,  involving, e.g, high-dimensional data of different modalities, intractable likelihood models, and/or small sample sizes. Section~\ref{sec: related_work} describes how LF2I is related to other work in this area.

At the heart of LF2I is the {\em Neyman construction of confidence sets}, albeit applied to a setting where the test statistic’s distribution is unknown. Frequentist confidence sets and their equivalence to hypothesis tests have a long history in statistics \citep{fisher1925,neyman1935CI,neyman1937inversion}.   While classical statistical procedures have significantly impacted fields like high-energy physics (see Section~\ref{sec: related_work}), most simulator-based methods lack theoretical guarantees for confidence sets beyond low-dimensional data and large-sample assumptions \citep{Feldman1998UnifyingApproach}. Implementing the Neyman construction for LFI is challenging not only because one cannot evaluate the likelihood, but also because one needs to test null hypotheses across the entire parameter space. While Monte Carlo and bootstrap methods estimate critical values and p-values from a batch of simulations at each null value $\theta_0$ \citep{mackinnon2009bootstrap,ventura2010bootstrap},  they become computationally infeasible for high-dimensional parameters. As a result,  practical implementations might rely on parametric assumptions or asymptotic theory~\citep{neymanpearson1928LR, wilks1938LRAsymptotic}. For instance, it is often assumed that the likelihood-ratio (LR) statistic follows a $\chi^2$ distribution, but this does not hold for irregular models or small sample sizes \cite{algeri2019searching, kieseler2022calorimetric, ho2021approximate}. This work seeks to quickly and accurately estimate critical values and coverage across the parameter space without knowing the test statistic distribution or relying on large-sample approximations.

The key insight behind LF2I is that the main quantities of interest in frequentist statistical inference --- test statistics, critical values, p-values and coverage of the confidence set --- are {\em distribution functions indexed by the (unknown) parameter $\theta$}, which generally vary smoothly over the parameter space $\Theta$.} As a result, one can leverage machine learning methods and data simulated in the neighborhood of a parameter to improve estimates of quantities of interest with fewer total simulations. 
Figure~\ref{fig:schematic_diagram} illustrates the general LF2I inference machinery, which is composed of three modular branches with separate functionalities:\\

\vspace{-0.2cm}\noindent {\bf i) The test statistic branch}  (Figure~\ref{fig:schematic_diagram} center and Section~\ref{sec:odds_based_tests}) \addTwo{uses a simulated set $\T$ to estimate a test statistic $\lambda(\D; \theta_0)$  for testing $H_{0, \theta_0}:\theta=\theta_0$ versus $H_{1, \theta_0}: \theta \neq \theta_0$.} We study the theoretical and empirical performance of LF2I confidence sets derived from likelihood-based test statistics learned via the odds function ${\mathbb{O}}(\X; \theta)$ (Equation~\ref{eq:odds_def}).\\
 
\vspace{-0.2cm}\noindent {\bf ii) The calibration branch} (Figure~\ref{fig:schematic_diagram} left and Section~\ref{sec:confidence_sets}) uses a left-out set $\mathcal{T}^{'}$ to estimate  critical values $C_{\theta_0}$  for every level-$\alpha$ test of $H_{0, \theta_0}$ via quantile regression of the estimated test statistic $\lambda(\D; \theta_0)$ on $\theta_0 \in \Theta$. Once we have estimated the quantile function $\widehat{C}_{\theta_0}$ \addTwo{indexed by $\theta_0$}, we can directly construct Neyman confidence sets  
\begin{equation}\label{eq:def_est_confset}
    \widehat{R}(\D):=\left\{\theta \in \Theta \, \middle| \, \lambda(\D;\theta) \geq \widehat{C}_{\theta} \right\}  
\end{equation}
that have approximate $(1-\alpha)$  finite-$n$ coverage \addTwo{for every value of $\theta \in \Theta$.} LF2I with critical values is amortized, meaning that once trained it can be evaluated on an arbitrary number of observations $D$. Alternatively, we can estimate p-values $p(D;\theta_0)$ for every test at $\theta=\theta_0$ with observed data $D$.\\
    
\vspace{-0.2cm}\noindent {\bf iii) The diagnostics branch} (Figure~\ref{fig:schematic_diagram} right and Section~\ref{sec:diagnostics}) uses a validation set $\mathcal{T}^{''}$
to assess the empirical coverage $\P_{\D | \theta}(\theta \in \widehat{R}(\D))$ of the constructed confidence sets $\widehat{R}(\D)$ \addTwo{ across the parameter space by regressing the indicator variable $W:=\mathbb{I}(\lambda(\D;\theta) \geq \widehat{C}_{\theta})$ on $\theta$.} \addTwo{The diagnostics branch is not part of the inference procedure itself. Its purpose is to provide an independent assessment of local (instance-wise) coverage of the final constructed confidence sets.}\\

\noindent The LF2I approach was first introduced in a conference proceeding \cite{dalmasso2020ACORE}. This preliminary version --- \texttt{ACORE} (Approximate Computation via Odds Ratio Estimation) --- uses a test statistic that maximizes odds over the parameter space. In this follow-up paper, we analyze the statistical and computational properties of LF2I, while also introducing a new test statistic --- the Bayesian Frequentist Factor (\texttt{BFF}) --- which is the Bayes Factor \citep{jeffreys_1935, jeffreys_1961} treated as a frequentist test statistic. We show that the validity of LF2I only depends on calibration, whereas its power depends on the test statistic's definition and its estimation quality. In addition to new theoretical results in Section~\ref{sec:theory}, we compare LF2I with approaches using Monte Carlo methods or Wilks' theorem (Section \ref{sec:GMM}), and we illustrate how our diagnostics can help scientists in choosing the best tool to handle nuisance parameters (Section \ref{sec:hep_example}). Finally, we construct confidence sets given a high-dimensional particle physics simulation where ABC approaches are neither computationally feasible nor valid (Section \ref{sec:muons}).

\begin{figure}[t!]
\centering
\includegraphics[width=0.6\textwidth]{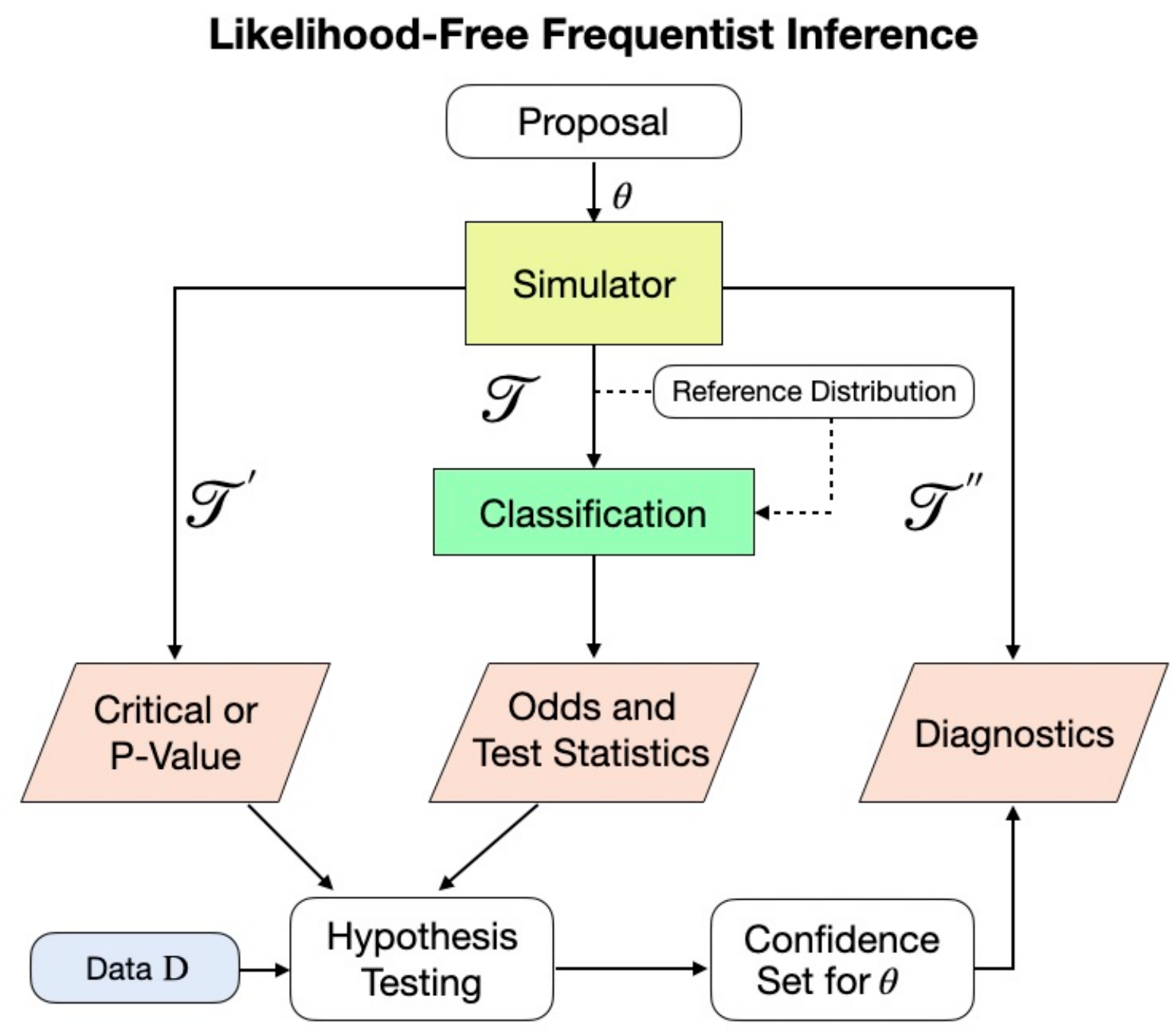}
\caption{\small {\bf The three-branch fully modular framework for likelihood-free frequentist inference (LF2I).} {\bf Center branch:} Draw a sample $\T$ of size $B$ from the simulator to estimate an arbitrary test statistic $\lambda(\D;\theta)$. Here we show how to do so by estimating the likelihood via the odds function $\mathbb{O}(\X;\theta)$. {\bf Left branch:} Draw a second sample $\T'$ of size $B'$ to estimate the critical values $C_{\theta}$ or p-values $p(\D; \theta)$ for all $\theta \in \Theta$. \textbf{Left $+$ Center:} Once data $D$ are observed, we can construct confidence sets $\widehat{R}(D)$ with finite-$n$ validity according to Equation~\ref{eq:est_conf_set}. {\bf Right branch:} The LF2I diagnostics branch independently checks whether the coverage $\P_{\D | \theta} (\theta \in \widehat{R}(\D))$ of the confidence set is indeed correct across the entire parameter space.}
\label{fig:schematic_diagram}
\vspace{-.2cm}
\end{figure}

\section{Statistical Inference in a Traditional Setting} 
\label{sec:classical_inf}
\addEJS{We now review the Neyman construction of confidence sets and the definitions of likelihood ratio and Bayes factor, before moving on to the details of the LF2I framework and its two instances, \ACORE and \texttt{BFF}.}

\paragraph{Equivalence of tests and confidence sets} A classical approach to constructing a confidence set for an unknown parameter $\theta \in \Theta$ is to invert a series of hypothesis tests \citep{neyman1937inversion}. Suppose that for each possible value $\theta_0 \in \Theta$, there is a level-$\alpha$ test $\delta_{\theta_0}$ of  
\begin{equation}
H_{0, \theta_0}: \theta=\theta_0 \ \ \mbox{versus} \ \ H_{1, \theta_0}: \theta \neq \theta_0.
\label{eq:test_Neyman}
\end{equation}
That is, a test $\delta_{\theta_0}$  where the type I error (the probability of erroneously rejecting a true null hypothesis $H_{0, \theta_0}$) is no larger than $\alpha$. For observed data $\D=D$, let $R(D)$ be the set of all parameter values $\theta_0 \in \Theta$ for which the test $\delta_{\theta_0}$ does not reject $H_{0,\theta_0}$. Then, by construction, the random set $R(\D)$ satisfies
\addEJS{\begin{equation*}
\P_{\D|\theta} \left( \theta \in R(\D) \right) \geq 1-\alpha \quad \forall \theta \in \Theta,
\end{equation*}}
which makes it a $(1-\alpha)$ {\em confidence set} for $\theta$. Similarly, we can define   tests with a desired significance level  by inverting  a confidence set with a certain coverage.

\paragraph{Likelihood ratio test} A general form of hypothesis tests that often leads to high power is the likelihood ratio test (LRT). Consider testing 
\begin{equation}   \label{eq:hypothesis_testing}
 H_0: \theta \in \Theta_0 \ \ \mbox{versus} \ \ H_1: \theta \in \Theta_1,
\end{equation}
where $\Theta_1=\Theta \setminus \Theta_0$.  For the {\em likelihood ratio (LR) statistic}, 
\begin{equation}
\label{eq::LRT}
\text{LR}(\D; \Theta_0) = 
\log \frac{\sup_{\theta \in \Theta_0}\mathcal{L}(\D;\theta)}{\sup_{\theta \in \Theta}\mathcal{L}(\D;\theta)},
\end{equation}
the LRT of hypotheses (\ref{eq:hypothesis_testing})
rejects $H_0$ when $\text{LR}(D; \Theta_0) < C$ for some constant $C$. Figure~\ref{fig:neyman_inversion_statistics} illustrates the construction of confidence sets for $\theta$ from level $\alpha$ likelihood ratio tests (\ref{eq:test_Neyman}). 
The critical value for each such test $\delta_{\theta_0}$ is 
\addEJS{$C_{\theta_0} = \sup \left\{C: \P_{\D|\theta_0} \left(\text{LR}(\D; \theta_0)< C  \right) \leq \alpha \right\}$}.

\paragraph{Bayes factor}
Let $\pi$ be a probability measure over the parameter space $\Theta$. The Bayes factor \citep{jeffreys_1935, jeffreys_1961} for comparing the hypothesis $H_0: \theta \in \Theta_0 $ to its complement, the alternative $H_1$, is the ratio of the marginal likelihood of the two hypotheses: 
\begin{equation}
\label{eq::BF}
\text{BF}(\D; \Theta_0) \equiv  \frac{\P(\D|H_0)}{\P(\D|H_1)} = \frac{\int_{\Theta_0} \mathcal{L}(\D;\theta) d\pi_0(\theta)}{\int_{\Theta_1}\mathcal{L}(\D;\theta) d\pi_1(\theta)},
\end{equation}
 where $\pi_0$ and $\pi_1$ are the restrictions of $\pi$ to the parameter regions $\Theta_0$ and $\Theta_1=\Theta_0^c$, respectively.
The Bayes factor is often used as a Bayesian alternative to significance testing, as it quantifies the change in the odds in favor of $H_0$ when going from the prior to the posterior: $\frac{\P(H_0|\D)}{\P(H_1|\D)} = \text{BF}(\D; \Theta_0) \frac{\P(H_0)}{\P(H_1)}$.\\

\begin{figure}[t!]
    \centering
    \includegraphics[width=0.3\textwidth]{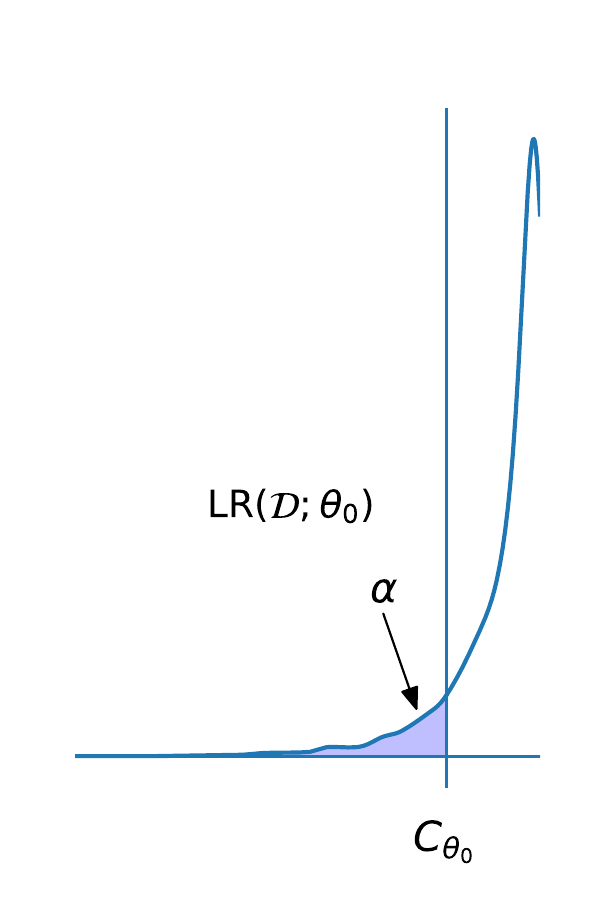}
    \includegraphics[width=0.27\textwidth]{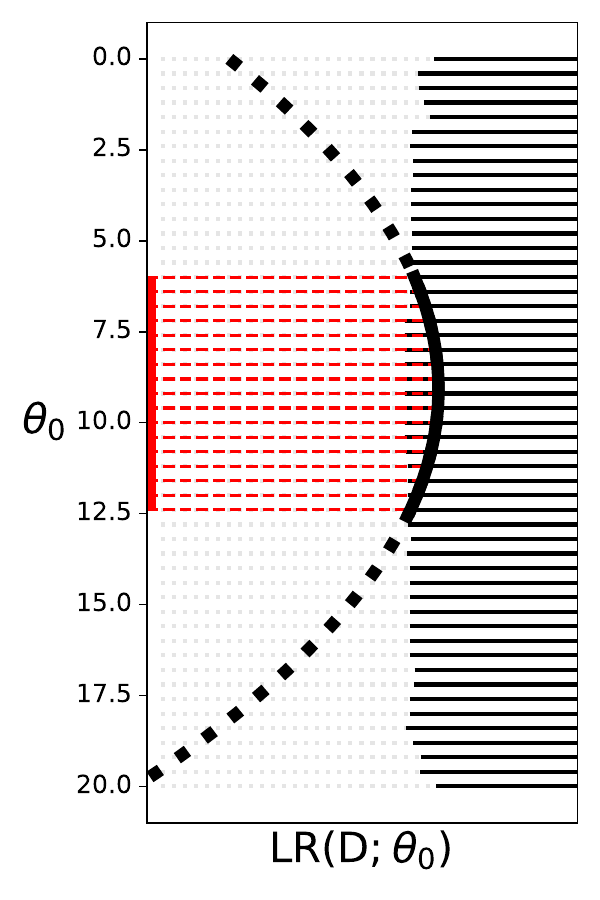}
    \vspace*{-3mm}
    \caption{\small {\bf Neyman construction of confidence sets by inverting hypothesis tests.} {\bf Left:} For each $\theta_0 \in \Theta$, we find the critical value $C_{\theta_0}$ that rejects the null hypothesis $H_{0,\theta_0}$ at level $\alpha$; that is, $C_{\theta_0}$ is the $\alpha$-quantile of the distribution of the test statistic under the null (a likelihood ratio $\text{LR}(\mathcal{D}; \theta_0)$ in this case). {\bf Right:} The horizontal solid lines represent acceptance regions for each $\theta_0 \in \Theta$. Suppose we observe data $D$. The confidence set for $\theta$ (red vertical solid line) consists of all $\theta_0$-values for which the observed test statistic $\text{LR}(D; \theta_0)$ (black curve) falls in the acceptance region.}
    \label{fig:neyman_inversion_statistics}
\end{figure}

\section{Likelihood-Free Frequentist Inference \addEJS{via Odds Estimation}}
\label{sec:LF2I}
In the typical LFI setting, we cannot directly evaluate the likelihood ratio $\text{LR}(\D; \Theta_0)$ or even the likelihood $\mathcal{L}(\D;\theta)$. 
\addEJS{In this work, we describe a version of LF2I that is based on odds estimation.} We assume that we have access to (i) a forward simulator $F_\theta$ to draw observable data, (ii) a reference distribution $G$ that does not depend on $\theta$, with larger support than $F_\theta$ for all $\theta \in \Theta$, and (iii) a probabilistic classifier to discriminates samples from $F_\theta$ and $G$.  

\subsection{\addEJS{Estimating an Odds Function across the Parameter Space}}\label{sec:learn_odds}
We start by generating a labeled sample $\T = \{(\theta_i,\X_i,Y_i)\}_{i=1}^B$ to compare data from $F_\theta$ with data from the reference distribution $G$. Here, $\theta \sim \pi_{\Theta}$ (a proposal distribution over $\Theta$), the ``label''  $Y \sim \text{Ber}(p)$,
$\X|(\theta,Y=1)\sim F_{\theta}$ and $\X|(\theta,Y=0)\sim G$. We then define the odds at $\theta$ and fixed $\x$ as
\begin{equation}\label{eq:odds_def}
{\mathbb{O}}(\x; \theta):=\frac{\P(Y=1|\theta,\x)}{\P(Y=0|\theta,\x)}.
\end{equation}
One way of interpreting ${\mathbb{O}}(\x; \theta)$ is to regard it as a measure of the chance that $\x$ was generated from $F_\theta$ rather than from $G$. That is, a large odds ${\mathbb{O}}(\x; \theta)$ reflects the fact that it is plausible that $\x$ was generated from $F_\theta$ (instead of $G$). We call $G$ a ``reference distribution'' as we are comparing $F_\theta$ for different $\theta$ with this distribution. \addEJS{Equation~\ref{eq:odds_def} is equivalent to the likelihood $p(\x|\theta)$ up to a normalization constant, as shown in \cite[Proposition 3.1]{dalmasso2020ACORE}.} The odds function ${\mathbb{O}}(\X; \theta)$  with $\theta \in \Theta$ as a parameter can be \revJMLR{estimated} with a probabilistic classifier, such as a neural network with a softmax layer, suitable for the data at hand. Algorithm~\ref{alg:joint_y} in Appendix~\ref{app:methods} summarizes our procedure for simulating a labeled sample $\T$. For all experiments in this paper, we use p=1/2 and $G=F_\X$, where $F_\X$ is the (empirical) marginal distribution of $F_\theta$ with respect to $\pi_\Theta$.

\subsection{Test Statistics based on Odds} \label{sec:odds_based_tests}
For testing $H_{0, \Theta_0}:\theta\in \Theta_0$ versus all alternatives $H_{1,\Theta_0}: \theta \notin \Theta_0$,
we consider two test statistics: \ACORE and \texttt{BFF}. Both statistics are based on $\mathbb{O}(\X; \theta)$, but whereas \ACORE eliminates the parameter $\theta$ by maximization, \BFF averages over the parameter space. 

\subsubsection{\ACORE by Maximization}\label{sec:ACORE}
The  \ACORE statistic \citep{dalmasso2020ACORE} for testing \revJMLR{Equation~\ref{eq:test_Neyman}} is given by 
\revJMLR{       \begin{align}
            \Lambda(\D; \Theta_0) &:= \log \frac{ \sup_{\theta \in \Theta_0}  \prod_{i=1}^n  \mathbb{O}(\X_i;\theta)}{\sup_{\theta \in \Theta}   \prod_{i=1}^n  \mathbb{O}(\X_i;\theta)} \nonumber \\
            	& =   \ \sup_{\theta_0 \in \Theta_0} \inf_{\theta_1 \in \Theta} \sum_{i=1}^n  \log \left( \Or(\X_i;\theta_0,\theta_1)\right),
     \label{eq:ACORE_statistic}
      \end{align}}
      where the odds ratio 
\revJMLR{      \begin{equation}\label{eq:oddsratio_def} \Or(\x;\theta_0,\theta_1):= \frac{\mathbb{O}(\x;\theta_0)}{\mathbb{O}(\x;\theta_1)} \end{equation}}
at $\theta_0,\theta_1 \in \Theta$ measures the plausibility that \revJMLR{a fixed $\x$} was generated from $\theta_0$ rather than $\theta_1$. We use  $\widehat \Lambda(\D; \Theta_0)$ to denote the \ACORE statistic based on $\T$ and estimated odds $\widehat{\mathbb{O}}(\X;\theta_0)$. When $\widehat{\mathbb{O}}(\X;\theta_0)$ is well-estimated for every $\theta$ and $\X$, $\widehat \Lambda(\D; \Theta_0)$ is the same as the $\text{LR}(\D; \Theta_0)$ in Equation~\ref{eq::LRT} \cite[Proposition 3.1]{dalmasso2020ACORE}.
        
\subsubsection{\BFF by Averaging}\label{sec:BFF}
Because the \ACORE statistics in  Equation~\ref{eq:ACORE_statistic} involves taking the supremum (or infimum) over $\Theta$, it may not be practical in high dimensions. Hence, in this work, we propose an alternative statistic for testing (\ref{eq:test_Neyman}) based on averaged odds:
\begin{equation}
  \tau(\D; \Theta_0) :=  \frac{\int_{\Theta_0}  \prod_{i=1}^n  \mathbb{O}(\X_i;\theta_0) d\pi_0(\theta)}{ \int_{\Theta_0^c}   \prod_{i=1}^n  \mathbb{O}(\X_i;\theta)  d\pi_1(\theta)} ,
    \label{eq:BFF_statistic}
\end{equation}
where $\pi_0$ and $\pi_1$ are the restrictions of the proposal distribution $\pi$ to the parameter regions $\Theta_0$ and $\Theta_0^c$, respectively. Let $\widehat \tau(\D; \Theta_0)$ denote estimates based on $\T$ and $\widehat{\mathbb{O}}(\theta_0; \x)$.  If the probabilities learned by the classifier are well estimated, then the 
 estimated averaged odds statistic $\widehat \tau(\D; \Theta_0)$ is exactly the Bayes factor:
 \begin{prop}[Fisher consistency] \label{prop::consistency}
\item \ \revJMLR{Assume that, for every $\theta \in \Theta$, \revEJS{$G$ dominates $\nu$}.}
If  $\ \widehat{\P}(Y=1|\theta,\x)=\P(Y=1|\theta,\x)$ for every $\theta$ and $\x$, 
 then $\widehat{\tau}(\D; \Theta_0)$ is the Bayes factor $\text{BF}(\D; \Theta_0)$.
\end{prop}
\noindent
In this paper, we are using the Bayes factor as a frequentist test statistic. Hence, our term {\em Bayes Frequentist Factor (BFF)} statistic for $\tau$ and $\widehat{\tau}$.

\subsection{Fast Construction of Neyman Confidence Sets}\label{sec:confidence_sets}

Instead of a costly MC or bootstrap hypothesis test 
of $H_0: \theta=\theta_0$ at each $\theta_0$ on a fine grid (see, e.g., \cite{mackinnon2009bootstrap} and \cite{ventura2010bootstrap}), we draw only one sample $\mathcal{T}'$ of 
size $B'$. We then estimate either the critical value $C_{\theta_0}$ via quantile regression (Section~\ref{sec:critical-value}), or the p-value $p(D;\theta_0)$ via probabilistic classification (Section~\ref{sec:pval}), for all $\theta_0 \in \Theta$ simultaneously. In Supplementary Material~H\footnote{Available at \href{https://lucamasserano.github.io/data/LF2I\_supplementary\_material.pdf}{https://lucamasserano.github.io/data/LF2I\_supplementary\_material.pdf}.}, we propose a practical strategy to choose the number of simulations $B'$ and the learning algorithm.

\subsubsection{The Critical Value via Quantile Regression}\label{sec:critical-value}
Algorithm \ref{alg:estimate_cutoffs} describes how to use quantile regression (e.g., \cite{meinshausen2006qrforest, koenker2017handbook}) to estimate the critical value $C_{\theta_0}$ for a level-$\alpha$ test of (\ref{eq:test_Neyman}) as a function of $\theta_0 \in \Theta$. To test a composite null hypothesis $H_0:\theta \in \Theta_0$ versus $H_1: \theta \in \Theta_1$, we use the cutoff $\hat{C}_{\Theta_0} := \inf_{\theta \in \Theta_0} \widehat C_{\theta}$. Although we originally proposed the calibration procedure for \texttt{ACORE}, the same scheme leads to a valid test (control of type I error as  the number of simulations $B' \rightarrow \infty$) for {\em any} test statistic $\lambda$  (Theorem~\ref{thm:convergenceCutoffs}). \addEJS{Remarkably, this holds even if the test statistic is not well estimated.} \addEJS{Note that in practice, we observe that the number of simulations $B^\prime$ needed to achieve correct coverage is usually much lower relative to $B$, the number of simulations needed to estimate the test statistic.} 
In addition, Algorithm~\ref{alg:estimate_cutoffs} does not rely on the observed data $D$ and is therefore amortized, meaning that once the test statistic and critical values have been estimated, we can compute confidence sets for any new data set without the need to retrain the model.

\begin{algorithm}[!t]
    \small
    \caption{Estimate critical values $C_{\theta_0}$  for a level-$\alpha$ test of  $H_{0, \theta_0}: \theta = \theta_0$ vs. $H_{1, \theta_0}: \theta \neq \theta_0$ for all $\theta_0 \in \Theta$ simultaneously}
    \begin{flushleft}
    {\bf Input:} simulator $F_\theta$; number of simulations $B'$; $\pi_\Theta$ (fixed proposal distribution over the parameter space); test statistic $ \lambda$; quantile regression estimator; level $\alpha \in (0,1)$
    {\bf Output:} estimated critical values $\widehat{C}_{\theta_0}$ for all $\theta_0 \in \Theta$
    \end{flushleft}
    \label{alg:estimate_cutoffs} 
    \begin{algorithmic}[1]
    \State Set $\T' \gets \emptyset$ 
    \For{i in $\{1,\ldots, B'\}$}
    \State Draw parameter $\theta_i \sim \pi_{\Theta}$
    \State Draw sample 
    $\X_{i,1},\ldots,\X_{i,n}  \stackrel{iid}{\sim}  F_{\theta_i}$
    \State Compute test statistic $\lambda_i \gets \lambda((\X_{i,1},\ldots,\X_{i,n});\theta_i) $
    \State $\T' \gets \T' \cup  \{(\theta_i,\lambda_{i})\}$  
    \EndFor
    \State  Use $\T'$ to learn 
    \revJMLR{the conditional quantile function} $\widehat{C}_\theta := \widehat{F}_{\lambda|\theta}^{-1}(\alpha | \theta)$
    via quantile regression of  $\lambda$ on $\theta$ \\
	\textbf{return}
	 $\widehat{C}_{\theta_0}$
	\end{algorithmic}
\end{algorithm}

\subsubsection{The P-Value via Probabilistic Classification} \label{sec:pval}
If the data $D$ are observed beforehand, then given any test statistic $\lambda$ we can alternatively compute p-values for each hypothesis $H_{0,\theta_0}:\theta=\theta_0$, that is, 
\addEJS{\begin{equation}
\label{eq::pvalue_regression}
p(D;\theta_0):=\P_{\mathcal{D} | \theta_0}\left( \lambda(\D;\theta_0)<  \lambda(D;\theta_0)\right).
\end{equation}}
The p-value $p(D;\theta_0)$ can be used to test hypothesis and create confidence sets for any desired  level $\alpha$. As detailed in Algorithm~\ref{alg:estimate_pvalues}, we can estimate it simultaneously for all $\theta \in \Theta$ by drawing a training sample $\T' = \{(Z_1,\theta_1),\ldots,$ $(Z_{B'},\theta_{B'})\}$ and using the random variable $Z:= \I\left( \lambda(\D;\theta)< \lambda(D;\theta)\right)$ as a label for each $\theta$. To test the composite null hypothesis $H_0: \theta \in \Theta_0$ versus $H_1: \theta \in \Theta_1$, we use
$$\hat p(D;\Theta_0):=\sup_{\theta \in \Theta_0}\hat p(D;\theta).$$
Note that there is a key computational difference between estimating p-values versus estimating critical values. The p-value is a function of both $\theta$ and the observed sample $D$ itself. As a result, Algorithm~\ref{alg:estimate_pvalues} has to be repeated for each observed $D$, making the computation of p-values non-amortized.

\subsubsection{Amortized Confidence Sets}\label{sec:confidence_set}
Finally, we construct an approximate confidence region for $\theta$ by taking the set
\begin{equation}\label{eq:est_conf_set} 
\hat R(D) = \left\{ \theta  \in \Theta \ \middle|  \lambda(D;\theta) \geq \hat C_{\theta}\right\}, 
\end{equation}
or, alternatively, 
\begin{equation}\label{eq:est_conf_set_pval}
\hat R(D) = \left\{ \theta  \in \Theta \, \middle|  \, \hat p(D; \theta) > \alpha \right\}. \end{equation}
See Algorithm~\ref{alg:conf_reg} in Appendix~\ref{app:confidence_sets} for details. As shown in \cite[Theorem 3.3]{dalmasso2020ACORE}, the random set $\hat R(\D)$ has nominal $(1-\alpha)$ coverage as $B' \rightarrow \infty$ regardless of the observed sample size $n$. \revJMLR{As noted in Section~\ref{sec:critical-value}, the confidence set in Equation~\ref{eq:est_conf_set} is fully {\em amortized}, meaning that once we have $\lambda(\D;\theta)$ and $\hat C_{\theta}$ as a function of $\theta \in \Theta$, we can perform inference on new data without retraining.}

\subsection{\addEJS{Diagnostics: Checking Coverage across the Parameter Space}}\label{sec:diagnostics}

\begin{algorithm}[!t]
    \small
    \caption{Estimate empirical coverage \addEJS{$ \pr_{\mathcal{D}|\theta}(\theta \in \widehat R(\D))$}, for all $\theta \in \Theta$.}
    \label{alg:estimate_coverage}
    \begin{flushleft}
        {\bf Input:} simulator $F_{\theta}$; number of simulations $B''$; $\pi_{\Theta}$ (fixed proposal distribution over parameter space); test statistic $\lambda$; level $\alpha$; critical values $\widehat{C}_\theta$; probabilistic classifier
        {\bf Output:} estimated coverage \addEJS{$ \widehat \pr_{\mathcal{D} | \theta}(\theta \in \widehat R(\D))$} for all $\theta  \in \Theta$
    \end{flushleft}
    \begin{algorithmic}[1]
    \State Set $\T'' \gets \emptyset$ 
    \For{i in $\{1,\ldots,B^{''}\}$}
    \State  Draw parameter $\theta_i \sim \pi_{\Theta}$
    \State Draw sample 
    $\D_i:=\{\X_{i,1},\ldots,\X_{i,n} \} \stackrel{iid}{\sim}  F_{\theta_i}$
    \State Compute  test statistic $\lambda_i \gets \lambda(\D_i;\theta_i)$
    \State Compute indicator  variable $W_i \gets \I\left(\lambda_i \geq \widehat{C}_{\theta_i}\right)$  \State $\T'' \gets \T'' \cup \{(\theta_i, W_i)\}$
    \EndFor
    \State  \addEJS{Use $\T''$ to learn $\widehat \pr_{\mathcal{D}|\theta'}(\theta' \in \widehat R(\D))$ across $\Theta$} by regressing  $W$ on $\theta$
    \State \textbf{return}  
    \addEJS{$\widehat \pr_{\mathcal{D}|\theta}(\theta \in \widehat R(\D))$}  
    \end{algorithmic}
\end{algorithm}

The LF2I framework has a separate module (``Diagnostics'' in Figure~\ref{fig:schematic_diagram}) for evaluating ``local'' goodness-of-fit in different regions of the parameter space $\Theta$. This estimates \addEJS{the coverage probability
$\P_{\D|\theta}(\theta \in \widehat{R}(\D))$ of confidence sets $\widehat{R}(\D)$ across the parameter space} via probabilistic classification. As detailed in Algorithm~\ref{alg:estimate_coverage}, we first generate a set of size $B''$ from the simulator: $\mathcal{T}'' = \{ (\theta_1, \mathcal{D}_1), \ldots, (\theta_{B''}, \mathcal{D}_{B''}) \}$. 
Then, for each sample $\mathcal{D}_i$, we check whether or not  the test statistic $\lambda_i$ is larger than the estimated critical value $\widehat{C}_{\theta_i}$ (the output from Algorithm~\ref{alg:estimate_cutoffs}). This is equivalent to computing a binary variable $W_i$ for whether or not the ``true'' value $\theta_i$ falls within the confidence set $\hat{R}(\D_i)$ (Equation~\ref{eq:est_conf_set}). Recall that the computations of the test statistic and the critical value  are amortized, meaning that we do not retrain algorithms to estimate these two quantities. The final step is to estimate empirical coverage as a function of $\theta$ by using $W$ as a label for each $\theta$. This estimation requires a new fit, but after training the probabilistic classifier, we can evaluate the estimated coverage anywhere in parameter space $\Theta$.

This diagnostic procedure locates regions in parameter space where estimated confidence sets might under- or over-cover; see Figures~\ref{fig:GMM_coverage}, \ref{fig:onoff_coverage} and \ref{fig:muons_results} for examples. 
 Note that standard goodness-of-fit techniques  for conditional densities
\citep{cooks2006posteriorquantile, bordoloi2010photoz, talts2018validating, schmidt2020photoz} only check for marginal coverage over $\Theta$.

\section{Theoretical Guarantees} 
\label{sec:theory}
We now prove consistency of the critical value and p-value estimation methods (Algorithms~\ref{alg:estimate_cutoffs} and~\ref{alg:estimate_pvalues}, respectively) and provide theoretical guarantees for the power of \texttt{BFF}.  We refer the reader to  Appendix~\ref{app:acore} for a proof \revJMLR{for finite $\Theta$} that 
the power of \ACORE converges to the power of LRT as $B$ grows (Theorem \ref{thm:convergenceLRT}). 

\revJMLR{In this section, $\P_{\D,\T'|\theta}$ denotes  the probability integrated over both $\D \sim F_\theta$ and $\T'$, whereas $\P_{\D|\theta}$ denotes integration over $\D \sim F_\theta$ only. For notational ease, we do not explicitly state again (inside the parentheses of the same expression) that we condition on $\theta$.}

\subsection{Critical Value Estimation}
\label{sec:theory_critical}

\revJMLR{
We start by showing that our procedure for choosing critical values leads to valid hypothesis tests (that is,  tests that control the type I error probability),  as long as the number of simulations $B'$ in Algorithm~\ref{alg:estimate_cutoffs} is sufficiently large. We assume that the null hypothesis is simple, that is,
$\Theta_0=\{\theta_0\}$ ---  which is the relevant setting for the Neyman construction of confidence sets in the absence of nuisance parameters.
  See Theorem \ref{thm:convergenceCutoffs} in Appendix \ref{app:theory} for results for composite  null hypotheses.} 
  
  \revJMLR{
  We assume  that the quantile regression estimator described in Section \ref{sec:critical-value}  is consistent in the following sense:}

\revJMLR{ \begin{Assumption}[Uniform consistency]
\label{assum:quantile_consistent_simple_null} 
Let $ F(\cdot|\theta)$ be the cumulative  distribution function of the test statistic $\lambda(\mathcal{D};\theta_0)$ conditional on $\theta$, where $\D \sim F_\theta$. Let $\hat F_{B'}(\cdot|\theta)$ be the estimated \addTwo{distribution function indexed by $\theta$}, implied by a quantile regression with a sample $\mathcal{T}^{'}$ of $B'$ simulations $\D \sim F_\theta$.
Assume that the quantile regression estimator is such that
$$\sup_{\lambda \in \mathbb{R}}|\hat F_{B'}(\lambda|\theta_0)-  F(\lambda|\theta_0)|\xrightarrow[B' \longrightarrow\infty]{\enskip P \enskip} 0.$$
\end{Assumption}}

\revJMLR{Assumption~\ref{assum:quantile_consistent_simple_null} holds, for instance, for quantile regression forests \citep{meinshausen2006qrforest}.}
\\

\revJMLR{Next, we show that 
Algorithm~\ref{alg:estimate_cutoffs} yields a valid hypothesis test as $B' \rightarrow \infty$. }

\revJMLR{
\begin{thm}
 \label{thm:valid_tests}
Let 
$C_{B'} \in \mathbb{R}$ be the 
critical value 
of the test based on  
an absolutely continuous statistic  $\lambda(\mathcal{D};\theta_0)$ chosen according to Algorithm~\ref{alg:estimate_cutoffs}
for a fixed $\alpha \in (0,1)$. If the quantile
estimator satisfies Assumption~\ref{assum:quantile_consistent_simple_null},
then, \revEJS{for every $\theta_0,\theta \in \Theta$,}
$$ \P_{\mathcal{D}|\theta_0,C_{B'}}(\lambda(\mathcal{D};\theta_0) \leq C_{B'})  \xrightarrow[B' \longrightarrow\infty]{\enskip a.s. \enskip}   \alpha,$$
where $\P_{\mathcal{D}|\theta_0,C_{B'}}$ denotes the probability integrated over $\mathcal{D}\sim F_{\theta_0}$ and conditional on the random variable $C_{B'}$.
\end{thm}
}

\revJMLR{If the convergence rate of the quantile regression estimator is known (Assumption \ref{assum:quantile_regression_rate}),  Theorem \ref{thm:valid_tests_rate} provides a finite-$B'$ guarantee on how far the type I error of the test will be from the nominal level.}

\revJMLR{ \begin{Assumption}[Convergence rate of the quantile regression estimator]
\label{assum:quantile_regression_rate}
Using the notation of Assumption \ref{assum:quantile_consistent_simple_null}, assume that the quantile regression estimator is such that
$$\sup_{\lambda \in \mathbb{R}}|\hat F_{B'}(\lambda|\theta_0)-  F(\lambda|\theta_0)|=O_P\left(\left(\frac{1}{B'}\right)^{r}\right)$$
for some $r>0$.
\end{Assumption}}

\revJMLR{
\begin{thm}
 \label{thm:valid_tests_rate}
With the notation and assumptions of Theorem \ref{thm:valid_tests}, and if  Assumption~\ref{assum:quantile_regression_rate} also holds,
then,
$$ |\P_{\mathcal{D}|\theta_0,C_{B'}}(\lambda(\mathcal{D};\theta_0) \leq C_{B'}) - \alpha| =O_P\left(\left(\frac{1}{B'}\right)^{r}\right).$$
\end{thm}
}

\subsection{P-Value Estimation}\label{sec:theory_pvalue_est}
\revEJS{Next we show} that the p-value estimation method described in Section \ref{sec:pval} is consistent. The results shown here 
 apply to any test statistic  $\lambda$.
 That is, these results are not restricted to \texttt{BFF}. 

We assume consistency in the sup norm of the regression method used to estimate the p-values:
\begin{Assumption}[Uniform consistency]
\label{assump:uniform_consistency}
The regression estimator used in Equation \ref{eq::pvalue_regression} is such that
$$\sup_{\theta \in \Theta_0} |\hat \E_{B'}[Z|\theta]- \E[Z|\theta]|  \xrightarrow[B' \longrightarrow\infty]{\enskip \text{a.s.} \enskip}  0.$$
\end{Assumption}

Examples of estimators that satisfy Assumption \ref{assump:uniform_consistency} include \cite{bierens1983uniform,hardle1984uniform,liero1989strong,girard2014uniform}.

The next theorem shows that the p-values obtained according to Algorithm \ref{alg:estimate_pvalues} converge to the true p-values. Moreover, the power of the tests obtained using the  estimated p-values converges to the power one would obtain if the true p-values could be computed.
\begin{thm}
\label{thm:pval_right_coverage}
Under Assumption \ref{assump:uniform_consistency} \revJMLR{and if $p(\D;\Theta_0)$ is \revEJS{an absolutely} continuous random variable then,} for every $\theta \in \Theta$,
$$\widehat p(D;\Theta_0) \xrightarrow[B' \longrightarrow\infty]{\enskip \text{a.s.} \enskip} p(D;\Theta_0)$$
and 
$$\P_{\D,\T'|\theta}(\hat p (\D;\Theta_0)\leq \alpha) \xrightarrow[B' \longrightarrow\infty]{} \P_{\D|\theta}(p (\D;\Theta_0)\leq \alpha).$$
\end{thm}

The next corollary shows that as $B' \longrightarrow \infty$, the tests obtained using the p-values from Algorithm \ref{alg:estimate_pvalues} have size $\alpha$.

\begin{Corollary}
\label{cor:pval_right_coverage}
Under Assumption \ref{assump:uniform_consistency} and if $F_{\theta}$ is continuous for every $\theta \in \Theta$ \revJMLR{and  $p(\D;\Theta_0)$ is an absolutely continuous random variable}, then
$$\sup_{\theta \in \Theta_0} \P_{\D,\T'|\theta}(\hat p (\D;\Theta_0)\leq \alpha) \xrightarrow[B' \longrightarrow\infty]{} \alpha.$$
\end{Corollary}

Under stronger assumptions about the regression method, it is also possible to derive rates of convergence for the estimated p-values.

\begin{Assumption}[Convergence rate of the regression estimator]
\label{assump:conv_reg_pval}
The regression estimator is such that
$$\sup_{\theta \in \Theta_0} |\hat \E[Z|\theta]- \E[Z|\theta]|=O_P\left(\left(\frac{1}{B'}\right)^{r}\right).$$
for some $r>0$.
\end{Assumption}

Examples of regression estimators that satisfy Assumption \ref{assump:conv_reg_pval} can be found in \cite{stone1982optimal,hardle1984uniform,donoho1994asymptotic,yang2017frequentist}.

\begin{thm}
\label{thm:pval_rate}
Under Assumption \ref{assump:conv_reg_pval}, 
$$|p(D;\Theta_0)-\hat p(D;\Theta_0)|=O_P\left(\left(\frac{1}{B'}\right)^{r}\right).$$
\end{thm}

\subsection{Power of \texttt{BFF}}\label{sec:power_bff}

In this section, we provide convergence rates for  \BFF and show that its power relates to
the \addEJS{integrated squared error 
\begin{equation}
  \mathcal{L}(\hat{\mathbb{O}}, \mathbb{O}) :=  \int \left( \hat{\mathbb{O}}(\x;\theta) - \mathbb{O}(\x;\theta) \right)^2 dG(\x) d\pi(\theta), 
   \label{eq::odds_loss}
\end{equation}
which measures how well we are able to estimate the odds function.
}

\revJMLR{We assume that we are testing a simple hypothesis $H_{0,\theta_0}:\theta=\theta_0,$
where $\theta_0$ is fixed,} and that $G(\x)$ is the marginal  distribution of $\revEJS{X \sim\ }F_{\theta}(\x)$ with respect to $\pi(\theta)$.  We also assume that $\x$ contains all observations; that is, $\mathbf{X}=\D$.  
In this case, the denominator of the average odds is
\begin{equation}
\begin{split}
    \int_{\Theta} \mathbb{O}(\x,\theta) d\pi(\theta) &= \int_{\Theta_1} \frac{p\cdot p(\x|\theta) }{(1-p)g(\x)} d\pi(\theta) \\
    &= \frac{p}{1-p}\int_{\Theta} \frac{p(\x|\theta) }{ \int_{\Theta} p(\x|\theta) d\pi(\theta) } d\pi(\theta) = \frac{p}{1-p},
\end{split}
\label{eq:bff_denom1}
\end{equation}
where $g$ is the density of $G$ with respect to $\nu$ and therefore there is no need to estimate the denominator in Equation \ref{eq:BFF_statistic}. 

We also assume that
 the odds and estimated odds are both bounded away from zero and infinity:
\begin{Assumption}[Bounded odds and estimated odds]
\label{assump:bounded_odds}
There exists $0 < m,M < \infty$ such that
for every $\theta \in \Theta$ and $\x \in \mathcal{X}$, 
$m \leq \mathbb{O}(\x;\theta), \hat{\mathbb{O}}(\x;\theta) \leq M$. 
\end{Assumption}

Finally, we  assume that the CDF of the power function of the test based on  the \BFF statistic $\tau$ in Equation~\ref{eq:BFF_statistic}   is smooth in a Lipschitz sense:
\begin{Assumption}[Smooth power function]
\label{assump:Lipschitz}
\revEJS{For every $\theta_0 \in \Theta$,} the  cumulative distribution function of $\tau(\D;\theta_0)$,  $F_\tau$, is Lipschitz with constant $C_L$, i.e., for every $x_1,x_2 \in \mathbb{R}$, $ |F_\tau(x_1) - F_\tau(x_2)| \leq C_L |x_1 - x_2|$.
\end{Assumption}

With these assumptions, we can relate the odds loss with the probability that the outcome of \BFF is different from the outcome of the test based on  the Bayes factor:

\revEJS{
\begin{thm}
\label{thm:bound_average}
For fixed $c \in \mathbb{R}$,
 let 
$\phi_{\tau;\theta_0}(\D)=\I\left(\tau(\D;\theta_0)<c\right)$ and
$\phi_{\hat\tau_B;\theta_0}(\D)=\I\left(\hat \tau_B(\D;\theta_0)<c\right)$
be the testing procedures for testing  $H_{0,\theta_0}:\theta=\theta_0$ based on $\tau$ and $\hat \tau_B$, respectively.
Under Assumptions \ref{assump:bounded_odds}-\ref{assump:Lipschitz}, for every $0<\epsilon<1$ and $\theta \in \Theta$,
   $$\int  \P_{\mathcal{D}|\theta,T}(\phi_{\tau;\theta_0}(\D) \neq \phi_{\hat\tau_B;\theta_0}(\D)) d\pi(\theta_0) \leq \frac{2MC_L\cdot \sqrt{ L(\hat{\mathbb{O}},\mathbb{O}) }}{\epsilon} + \epsilon,$$
   where $T$ denotes the realized training sample   $\mathcal{T}$  and $\P_{\mathcal{D}|\theta,T}$ is the probability measure integrated over the observable data $\D \sim F_\theta$,  but conditional on the train sample used to create the test statistic. 
\end{thm}
}

\revEJS{Theorem \ref{thm:bound_average} demonstrates that, on average (over $\theta_0 \sim \pi$), the probability that hypothesis tests based on the \BFF statistic versus the Bayes factor lead to different conclusions is bounded by the integrated odds loss.} This result is valuable because the integrated odds loss is easy to estimate in practice, and hence provides us with a practically useful metric. For instance, the integrated odds loss can serve as a natural criterion for selecting the ``best'' statistical model out of a set of candidate models with different classifiers, for tuning model hyperparameters, and for evaluating model fit.\\ 

Next, we provide rates of convergence of the test based on \BFF to the test based on the Bayes factor.
 We assume that the chosen probabilistic classifier has the following rate of convergence:
\begin{Assumption}[Convergence rate of the probabilistic classifier]
\label{assump:modelclass}
The probabilistic classifier trained with $\T$, $\hat \P(Y=1|\x,\theta)$  
is such that
$$ \E_{\T} \left[ \int \left( \hat \P(Y=1|\x,\theta) - \P(Y=1|\x,\theta) \right)^2 dH(\x, \theta) \right] = O\left(B^{-\kappa / (\kappa+d)} \right),$$
for some $\kappa > 0$ and $d>0$, where 
$H(\x, \theta)$ is a  measure over $\mathcal{X} \times \Theta$.
\end{Assumption}

Typically,  $\kappa$ relates to the smoothness of $\P$, while $d$ relates to the number of covariates of the classifier --- in our case, the number of parameters plus the number of features. In Supplementary Material~I, we provide some examples where Assumption \ref{assump:modelclass} holds.

We also assume that the density of the product measure $G \times \pi$ is bounded away from infinity.
\begin{Assumption}[Bounded density]
\label{assump:bounded_measure}
$H(\x, \theta)$ dominates $H' := G \times \pi$, and the density of $H'$  with respect to $H$, denoted by $h'$, is such that
there exists $\gamma>0$ with
$h'(\x,\theta) < \gamma,\ \forall \x \in \mathcal{X}, \theta \in \Theta$.
\end{Assumption}

If the probabilistic classifier has the convergence rate given by Assumption \ref{assump:modelclass}, then the \revEJS{average} probability that hypothesis tests based on the \BFF statistic versus the Bayes factor goes to zero  has the rate given by the following theorem.
\revEJS{
\begin{thm}
\label{thm:bound_different_tests_eps}
Let $\phi_{\tau;\theta_0}(\D)$ and
$\phi_{\hat\tau_B;\theta_0}(\D)$ be as in Theorem \ref{thm:bound_average}.
Under Assumptions \ref{assump:bounded_odds}-\ref{assump:bounded_measure},
there exists $K'>0$ such that, for any $\theta \in \Theta$,
$$\int  \P_{\mathcal{D},\T|\theta}(\phi_{\tau;\theta_0}(\D) \neq \phi_{\hat\tau_B;\theta_0}(\D))d\pi(\theta_0) \leq  K' B^{-\kappa / (4(\kappa+d))}.$$
\end{thm}
}

\revEJS{
\begin{Corollary}
\label{coroll:power}
Under Assumptions \ref{assump:bounded_odds}-\ref{assump:bounded_measure}, there exists $K'>0$
such that, for any $\theta \in \Theta$,
$$\int \P_{\mathcal{D},\T|\theta}(\phi_{\hat \tau_B;\theta_0}(\D)=1) d\theta_0 \geq \int \P_{\mathcal{D},\T|\theta}(\phi_{\tau;\theta_0}(\D)=1)d\theta_0  -   K' B^{-\kappa / (4(\kappa+d))}.$$
\end{Corollary}
}

Corollary \ref{coroll:power}  tells us that the \revEJS{average} power of the \BFF test is close to the \revEJS{average}  power of the exact Bayes factor test.  
 \revEJS{This result also} implies that \BFF converges to the most powerful test in the Neyman-Person setting, where the Bayes factor test is equivalent to the LRT.

\section{Handling Nuisance Parameters}
\label{sec:nuisance}
In most applications, \revJMLR{we} only \revJMLR{have} a small number of parameters  \revJMLR{that} are of primary interest. \revJMLR{The other parameters in the model are usually referred to as nuisance parameters.}  In this setting, \revJMLR{we decompose the parameter space as} $\Theta=\Phi \times \Psi$, where $\Phi$ contains the parameters of interest, and $\Psi$ contains nuisance parameters. \revJMLR{Our} goal is to construct a confidence set for $\phi \in \Phi$.  
To guarantee frequentist coverage by Neyman's inversion technique, \revJMLR{however,}
one needs to test null hypotheses of the form $H_{0, \phi_0}: \phi=\phi_0$ by comparing the test statistics to the cutoffs
$\hat{C}_{\phi_0} := \inf_{\psi \in \Psi} \widehat C_{(\phi_0,\psi)}$ (Section~\ref{sec:critical-value}). That is, one needs to control the type I error at each $\phi_0$ for  {\em all} possible values of the nuisance parameters.
Computing such infimum can be numerically unwieldy, especially if the number of nuisance parameters is large \citep{Boom2019nuisance, zhu2020nuisance}. \revJMLR{Below we propose approximate schemes for handling nuisance parameters:}

In \texttt{ACORE}, 
we use a hybrid resampling or ``likelihood profiling'' method \citep{chuang2000hybrid,feldman2000profiling,Sen2009NuisanceParameters} to circumvent unwieldy numerical calculations as well as to reduce computational cost. For each $\phi$ \revJMLR{(on a fine grid over $\Phi$)}, we first compute the
``profiled'' value $$\widehat{\psi}_{\phi}=\arg \max_{\psi \in\Psi} \prod_{i=1}^n \widehat{\mathbb{O}} \left(\x^{\rm obs}_i;(\phi,\psi) \right),$$ 
which (because of the odds estimation)  is an approximation of the maximum likelihood estimate of $\psi$  at the parameter value $\phi$ for observed data $D$.  By definition, the estimated \ACORE test statistic \revJMLR{for the hypothesis $H_{0, \phi_0}: \phi=\phi_0$} is  exactly given by   
$\widehat \Lambda(\mathcal{D}; \phi_0)=\widehat \Lambda(\mathcal{D}; (\phi_0, \widehat{\psi}_{\phi_0}))$. 
However, rather than comparing this statistic to $\hat{C}_{\phi_0} $, we use the hybrid cutoff 
\revJMLR{
\begin{equation}
     \hat C'_{\phi_0}:= \widehat{F}^{-1}_{\widehat \Lambda \left(\D;\phi_0 \right) \big| \left(\phi_0,\widehat \psi_{\phi_0} \right)}\left(\alpha \, \middle\vert \, \phi_0,\widehat \psi_{\phi_0} \right),  \label{eq:hybrid_cutoff}
 \end{equation}}
 where $\widehat{F}^{-1}$ is obtained  
   via a quantile regression as in Algorithm \ref{alg:estimate_cutoffs}, but using a training sample $\T'$ generated at {\em fixed} $\hat{\psi}_{\phi_0}$  \revJMLR{(that is, we run Algorithm \ref{alg:estimate_cutoffs} with the proposal distribution $\pi'((\phi,\psi)) \propto \pi(\phi)\times  \delta_{\hat{\psi}_{\phi}}(\psi),$
   where $\delta_{\hat{\psi}_{\phi}}(\psi)$ is a point mass distribution at $\hat{\psi}_{\phi}$)}. Alternatively,  one can compute \addEJS{the p-value
 \begin{equation}
 \widehat p(D;\phi_0):=\widehat \E \left[ \I\left( \hat{\Lambda} \left(\D;  \phi_0  \right) < \hat \Lambda \left(D; \phi_0  \right) \right)\, \middle\vert \, \phi_0,\widehat \psi_{\phi_0}  \right] \label{eq:hybrid_pvalue}
 \end{equation}
via probabilistic classification} as in Algorithm \ref{alg:estimate_pvalues}, but with $\T'$ simulated at fixed $\hat{\psi}_{\phi_0}$ \revJMLR{(that is, we run Algorithm \ref{alg:estimate_pvalues} with the proposal distribution $\pi'((\phi,\psi)) \propto \pi(\phi)\times  \delta_{\hat{\psi}_{\phi}}(\psi)$}. Hybrid methods do not always control $\alpha$, but they are often a good approximation that lead to robust results
\citep{Aad2012HiggsBoson,qian2016physics}. We refer to \texttt{ACORE} approaches based on Equation~\ref{eq:hybrid_cutoff} or Equation~\ref{eq:hybrid_pvalue} as  ``\texttt{h-ACORE}'' approaches.

In contrast to \texttt{ACORE}, 
the \texttt{BFF} test statistic averages (rather than maximizes) over nuisance parameters. 
Hence, instead of adopting a hybrid resampling scheme to handle nuisance parameters, we approximate p-values and critical values, in what we refer to as ``\texttt{h-BFF}'',  by using the marginal model of the data \revEJS{$\D$} at a parameter of interest $\phi$:
\revEJS{$$  \mathcal{ \widetilde L}(D;\phi) = \int_{\psi \in \Psi} \mathcal{L}(D;\theta)  \, d\pi(\psi) . $$}
We implement such a scheme by first  drawing the train sample  $\mathcal{T}^{'}$ from the entire parameter space $\Theta = \Phi \times \Psi$, and then applying quantile regression (or probabilistic classification) using $\phi$ only.

Algorithm~\ref{alg:conf_reg_nuisance} details our  construction of \ACORE and \BFF confidence sets  when calibrating critical values under the presence of nuisance parameter  (construction via p-value estimation is analogous).  In Section~\ref{sec:hep_example}, we demonstrate how our diagnostics branch can shed light on whether or not the final results have adequate frequentist coverage.

\section{Experiments}
\label{sec:examples}
\addEJS{We analyze the empirical performance of the LF2I framework under different problem settings: unknown null distribution of (known) test statistic (Section~\ref{sec:GMM}); nuisance parameters (Section~\ref{sec:hep_example}); intractable likelihood and high-dimensional data (Section~\ref{sec:muons}).

We use the cross-entropy loss (Eq.~\ref{eq::CE_loss}) when estimating the odds function in Equation~\ref{eq:odds_def} and the empirical coverage probability as in Section~\ref{sec:diagnostics} via probabilistic classification. Moreover, we use the pinball loss \citep{koenker2017handbook} when estimating critical values as in Section~\ref{sec:critical-value} via quantile regression.
}

\subsection{Gaussian Mixture Model: Unknown  Null Distribution}\label{sec:GMM}
A common practice in LFI is to first estimate the likelihood and then assume that the LR statistic is approximately $\chi^2$ distributed according to Wilks' theorem \citep{Drton2009LRSing}. However, in settings with small sample sizes or irregular statistical models, such approaches may lead to confidence sets with incorrect coverage; it is often difficult to identify exactly when that happens, and then know how to recalibrate the confidence sets. (See \cite{algeri2019searching} for a discussion of all conditions needed for Wilks' theorem to apply, which are often not realized in practice.)\\
\\
The Gaussian mixture model (GMM) is a classical example where the 
LR statistic is known \addEJS{but its null distribution is unknown in finite samples}. Indeed, the development of valid statistical methods for GMM is an active area of research \citep{redner1981gmm, mclachlan1987gmm, dacunha1997gmm, chen2009gmm, wasserman2020universal}. Here we consider a one-dimensional Normal mixture with unknown mean but known unit variance:
\begin{align*}
    X \sim 0.5 N(\theta, 1) + 0.5 N(-\theta, 1),
\end{align*}
where the parameter of interest $\theta \in \Theta =[0,5]$. \addTwo{In this example, the LRT statistic is not estimated but computed exactly. The goal is to} analyze three different approaches for estimating the critical value $C_{\theta_0}$ of a level-$\alpha$ LRT of the hypothesis test $H_{0,\theta_0}:\theta=\theta_0$, for different $\theta_0 \in \Theta$, \addTwo{in a setting where we have removed potential effects of estimation errors in the test statistic:}
\begin{figure}[t!]
    \centering
    \includegraphics[width=1\textwidth]{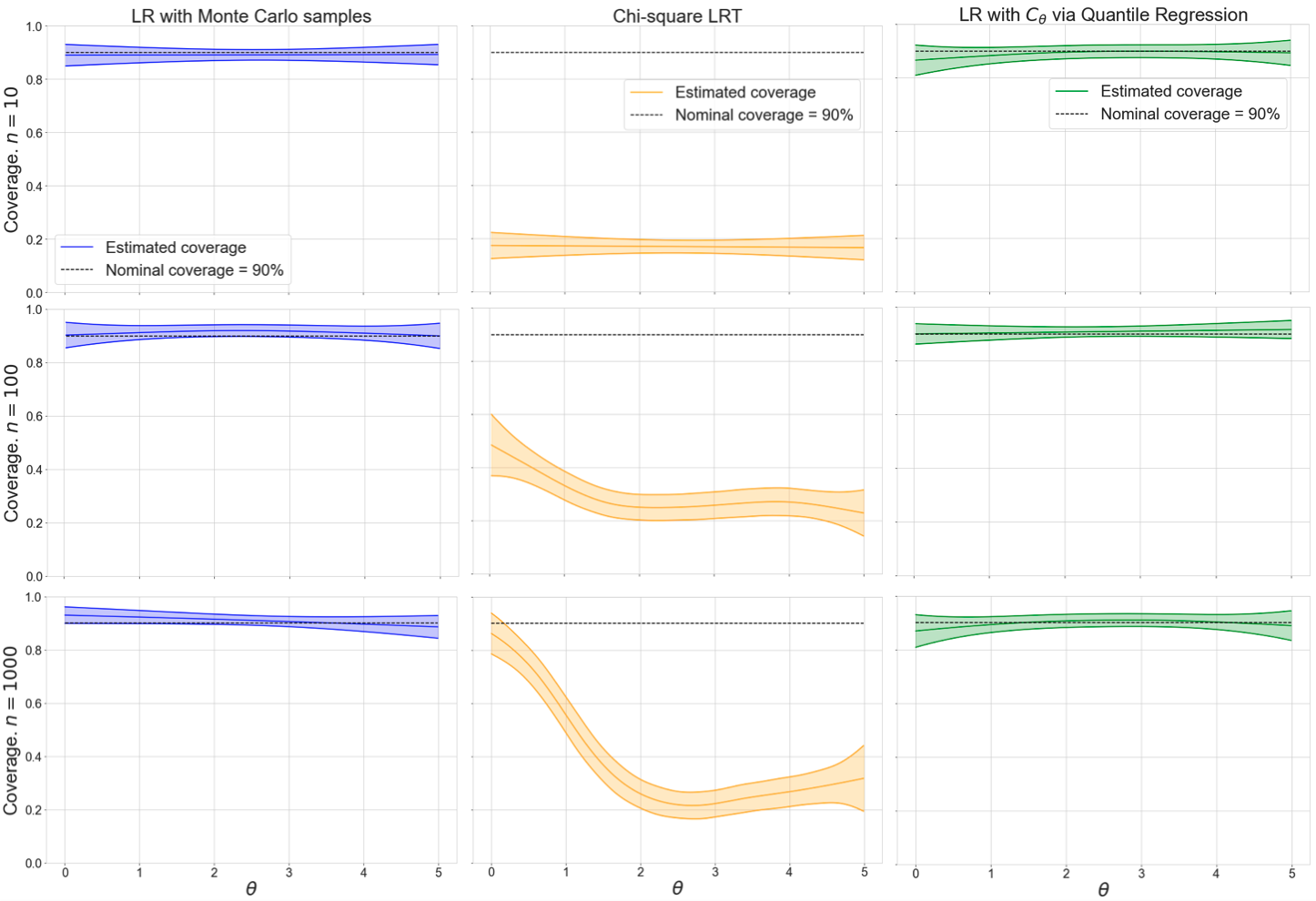}
    \caption{\small {\bf GMM with unknown null distribution}. Each panel shows the estimated coverage across the parameter space of 90\% confidence sets for $\theta$. Rows represent experiments with different observed sample sizes: $n=10,100, 1000$ (top, center, bottom). Columns represent three different approaches. {\bf Left:} ``LR with Monte Carlo samples'' achieves nominal coverage everywhere but is computationally expensive, especially in higher dimensions. {\bf Center:} ``Chi-square LRT'' clearly under-covers, i.e. confidence sets are not valid even for large $n$, other than at $\theta=0$ where the mixture collapses to one Gaussian. {\bf Right:} ``LR with $C_{\theta_0}$ via quantile regression'' returns finite-sample confidence sets with the nominal coverage of $90\%$ for all values of $\theta$, but using a total of 1000 simulations, instead of a MC sample of 1000 simulations at each grid point. 
    }
    \label{fig:GMM_coverage}
\end{figure}

\begin{itemize}
    \item
    ``LR with Monte Carlo samples'', where we
    \revJMLR{draw 1000 simulations at each point $\theta_0$ on a fine grid over $\Theta$ and take $C_{\theta_0}$ to be the $1-\alpha$ quantile of the distribution of the LR statistic, computed using the MC samples at each fixed $\theta_0$. This approach is often just referred to as MC hypothesis testing.}
     \item ``Chi-square LRT'', where we {\em assume} that $-2\text{LR}(\D;\theta_0) \sim \chi^2_1$, and hence take \revJMLR{$-2C_{\theta_0}$} to be the same as the upper $\alpha$ quantile of a $\chi^2_1$ distribution. 
    \item ``LR with $C_{\theta_0}$ via quantile regression'', where we  estimate $C_{\theta_0}$ via quantile regression (Algorithm~\ref{alg:estimate_cutoffs}) \revJMLR{based on a total of $B'=1000$ simulations of size $n$ sampled uniformly on $\Theta$.} 
\end{itemize}
\revJMLR{We then construct confidence sets by inverting the hypothesis tests, and finally assess their conditional coverage with the diagnostic branch of the LF2I framework (Algorithm~\ref{alg:estimate_coverage} with $B''=1000$).}

\revJMLR{Figure \ref{fig:GMM_coverage} shows LF2I diagnostics for the three different approaches when the observed sample size \addTwo{(i.e., the number of observations from each unknown $\theta$)} is $n=10, 100, 1000$. 
Confidence sets from ``Chi-square LRT'' are clearly not valid at any $n$, which shows that Wilks' theorem does not apply in this setting. The only exception arises when $n$ is large enough and $\theta$ approaches 0, in which case the mixture reduces to a \addTwo{{\em unimodal}} Gaussian whose LR statistic has a known limiting distribution (see bottom center panel of  Figure~\ref{fig:GMM_coverage}). On the other hand, ``LR with $C_{\theta_0}$ via quantile regression'' returns valid finite-sample confidence sets with conditional coverage equivalent to ``LR with Monte Carlo samples''. A key difference between the LF2I and MC methods is that the LF2I results are based on 1000 samples in total, whereas the MC results are based on 1000 MC samples at each $\theta_0$ on a grid. The latter approach quickly becomes intractable in higher parameter dimensions and larger scales.}\\

\noindent \addTwo{In Appendix~\ref{app:critical_vals}, we show that critical values are clearly non-constant across the parameter space, which also provides insight as to why assumptions of a pivotal test statistic (e.g., a $\chi^2$-distributed test statistic asymptotically, or calibration based on a single point in the parameter space \cite{warne2024generalised}) do not yield correct coverage.} \revJMLR{Supplementary Material~J gives details on the specific quantile regressor (for Algorithm~\ref{alg:estimate_cutoffs}) and probabilistic classifier (for Algorithm~\ref{alg:estimate_coverage}) used in Figure~\ref{fig:GMM_coverage}, and presents extensions of the above experiments to  confidence sets via p-value estimation and asymmetric mixtures.}

\subsection{\addEJS{Poisson Counting Experiment: Nuisance Parameters and Diagnostics}}
\label{sec:hep_example}

\begin{figure}[t!]
    \centering
    \includegraphics[width=1.0\textwidth]{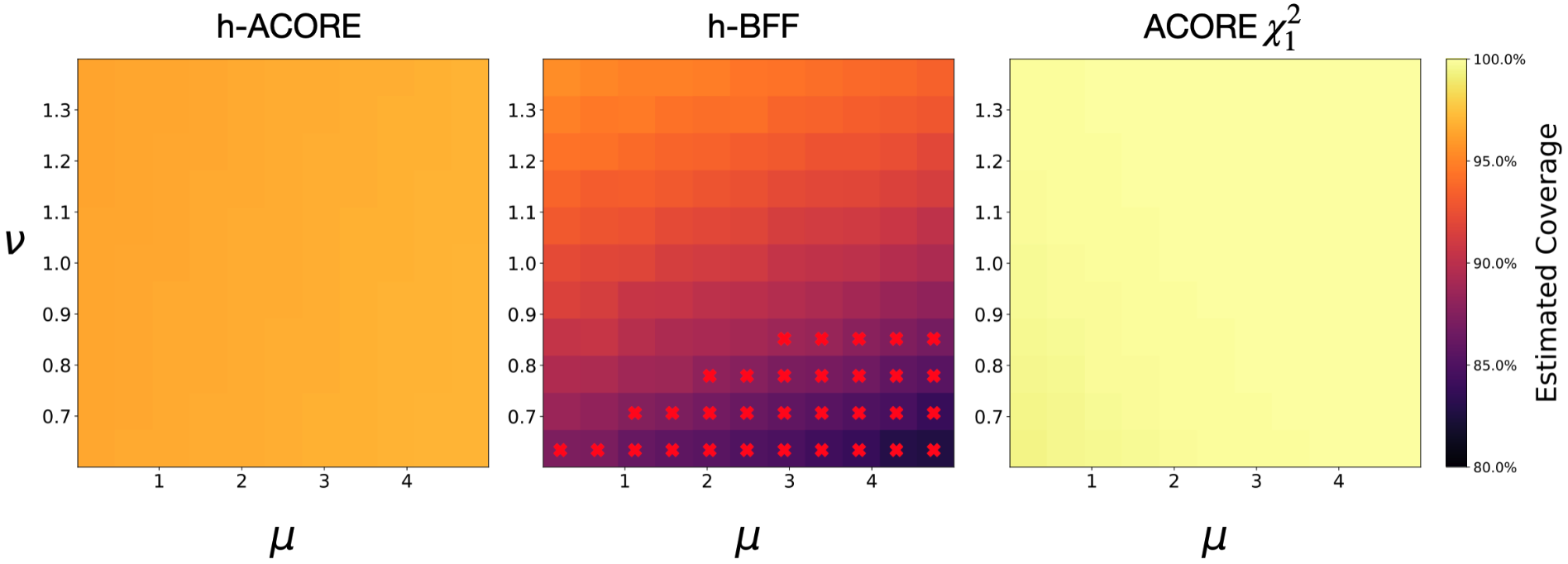}
    \caption{\small \addEJS{{\bf Poisson counting experiment with nuisance parameters.} The diagnostics branch provides guidance as to which LFI approach to use for the problem at hand by pinpointing regions of the parameter space $\Theta$ where inference is unreliable. The panels show empirical coverage as a function of both $\mu$, the parameter of interest, and $\nu$, the nuisance parameter. Nominal coverage is $90\%$. {\bf Left:} h-ACORE, which uses profiled likelihoods, is overly conservative in terms of actual coverage ($\approx 96\%$) across $\Theta$. {\bf Center:} h-BFF, which marginalizes over $\nu$, under-covers in several regions (red crosses). {\bf Right:} ACORE $\chi_1^2$, which uses cutoffs from the chi-square distribution, has almost no constraining power, yielding empirical coverage close to $100\%$ everywhere.}}
    \label{fig:onoff_coverage}
\end{figure}

Hybrid methods, which maximize or marginalize over nuisance parameters, do not always control the type I error of statistical tests. For small sample sizes, there is no theorem as to whether profiling or marginalization of nuisance parameters will give better frequentist coverage for the parameter of interest \citep[Section 12.5.1]{cousins2018lectures}. In addition, most practitioners consider a thorough check of frequentist coverage to be impractical \citep[Section 13]{cousins2018lectures}. In this example, we apply the hybrid schemes from Section~\ref{sec:nuisance} to a high-energy physics (HEP) counting experiment \citep{lyons2012counting, cowan2011formulaeInferenceHEP, cowan2012counting, cousins2008evaluation, heinrich2022learning} with nuisance parameters\revJMLR{, which is a simplified version of a real particle physics experiment where the true likelihood function is not known}. We illustrate how our diagnostics can guide the analyst and provide insight into which method to choose for the problem at hand.\\
\\
Consider a ``Poisson counting experiment'' where particle collision events are counted under the presence of both an uncertain background process and a (new) signal process. The goal is to estimate the signal strength.  To avoid identifiability issues, the background rate is estimated separately by counting the number of events in a control region where the signal is believed to be absent. Hence, the observable data $\mathbf{X} = (N_b, N_s)$ contain two measurements, where $N_b \sim {\rm Pois}(\nu\tau b)$ is the number of events in the control region, and $N_s \sim {\rm Pois}(\nu b + \mu s)$ is the number of events in the signal region. Our parameter of interest is the signal strength 
\addEJS{$\mu$, whereas the scaling factor for the background $\nu$ is a nuisance parameter. The hyper-parameters $s$ and $b$ indicate the nominally expected counts from signal and backgrounds, and $\tau$ describes the relationship in measurement time between the two processes. We treat the three hyper-parameters as known with values $s=15$, $b=70$, $\tau=1$, respectively. The hyper-parameters move the model away from the Gaussian limiting regime and make the relationship between data and parameters more complicated \cite{heinrich2022learning}.\\
\\
We compare the hybrid methods \texttt{h-ACORE} and \texttt{h-BFF} with \texttt{ACORE $\chi^2_1$} (which uses cutoffs from the chi-square distribution). We learn the odds using a QDA classifier with $B=100{,}000$ and estimate critical values for the hybrid methods via quantile gradient boosted trees with $B^{\prime}=10{,}000$. We evaluate the different methods on a separate set of size $B^{\prime\prime}=1{,}000$ by estimating coverage and measuring the length of confidence sets for each of the simulated samples.

\begin{figure}[h!]
  \begin{minipage}[c]{0.5\textwidth}
    \includegraphics[width=\textwidth]{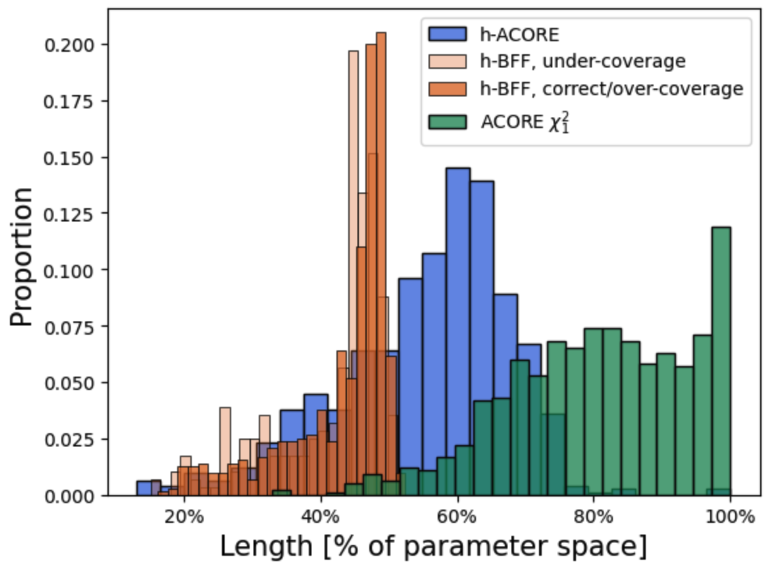}
  \end{minipage}\hfill
  \begin{minipage}[c]{0.45\textwidth}
    \caption{\small \addEJS{{\bf Constraining power.} Relative size of the confidence sets constructed in Section~\ref{sec:hep_example}. \addTwo{\texttt{ACORE} $\chi_1^2$ and \texttt{h-ACORE} yield the widest intervals (they are indeed overly conservative according to Figure~\ref{fig:onoff_coverage})}. \texttt{h-BFF} provides tighter confidence sets, but their size cannot be trusted when the method under-covers. LF2I diagnostics can identify the parameter regions where the approach is not valid (red crosses in Figure~\ref{fig:onoff_coverage}). \addTwo{The dark-orange histogram reports \texttt{h-BFF} results after removing those points.}}}
    \label{fig:onoff_length}
  \end{minipage}
\end{figure}

Figure~\ref{fig:onoff_coverage} shows the estimated coverage as a function of both $\mu$ and $\nu$. Confidence sets are considered to be valid when they achieve the nominal coverage level regardless of the true value of \textit{both} the parameter of interest and the nuisance parameters. Both \texttt{h-ACORE} and \texttt{ACORE $\chi_1^2$} are overly conservative across the whole parameter space, while \texttt{h-BFF} under-covers in regions of high signal strength and low background. These results are consistent with the length of the corresponding confidence sets shown in Figure~\ref{fig:onoff_length}: \texttt{h-ACORE} and \texttt{ACORE $\chi_1^2$} are overly conservative, with the former being almost uninformative for the majority of evaluation samples. On the other side, while \texttt{h-BFF} seems to provide tighter parameter constraints, their length can be trusted only in regions where the method has coverage at least equal to the nominal level. Our LF2I diagnostic branch can pinpoint the regions of the parameter space where inference is reliable or not.}

\subsection{\addEJS{Muon Energy Estimation: Intractable Likelihood and High-Dimensional Data}}\label{sec:muons}

\addEJS{We now showcase LF2I on a high-energy physics application with intractable likelihood and very high-dimensional data. The goal is to estimate the energy of muons using a high-granularity calorimeter in a particle collider experiment. Muons are subatomic particles that have proven to be excellent probes of new physical phenomena: their detection and measurement has enabled several crucial discoveries in the last few decades, including the discovery of the Higgs boson \cite{augustin1974discovery, herb1977observation, cdf1995observation, aad2012observation, chatrchyan2012observation}. Traditionally, the energy of a muon is determined from the curvature of its trajectory in a magnetic field, but curvature-based measurements have proven to be insufficiently precise at high energies. Recently, muon energy measurements based on their radiative losses in a dense, finely segmented calorimeter (Figure~\ref{fig:muons_results}, left) have been shown to be a feasible alternative \cite{kieseler2022calorimetric, dorigo2022deep}.\\
\\
\addEJS{In this application, the dimensionality of one data point $\x$ (a 3D image) is of the order of $\approx 50{,}000$ and the observed sample size is $n=1$ (as each unique data point is the output of one experiment with a specific parameter of interest $\theta$). In total, we have available $886{,}716$ 3D “image” inputs $\mathbf{x}$ with corresponding scalar muon energies $\theta$. The data are obtained by accurately mimicking particle showers with \textsc{GEANT4} \citep{agostinelli2003geant4}, a high-fidelity simulator that has been calibrated for decades and is trusted to incorporate all the dynamics of the Standard Model of particle physics.} The data are available at \cite{kieseler10preprocessed}.\\

\begin{figure}[t!]
    \centering
    \includegraphics[width=1.0\textwidth]{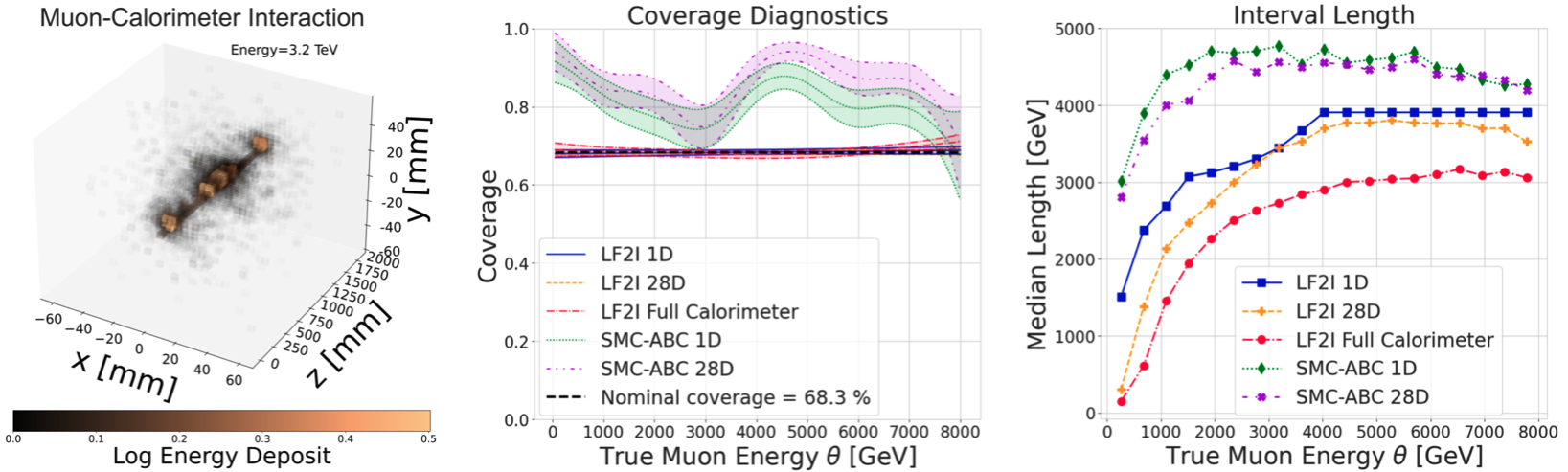}
    \caption{\small \addEJS{{\bf Muon energy estimation.} LF2I guarantees nominal coverage and yields smaller confidence intervals relative to SMC-ABC. {\bf Left:} Data point example of a muon with incoming energy $\theta \approx 3.2$ TeV entering a calorimeter with $32 \times 32 \times 50$ cells. {\bf Center:} LF2I (blue, orange, red in the right two panels) achieves coverage at the nominal level (68.3\%), whereas SMC-ABC (green and purple) is consistently over-covering across the parameter space. {\bf Right:} Median lengths of constructed intervals. While being extremely computationally intensive, SMC-ABC has also the least constraining power regardless of the data set used. SMC-ABC on the full calorimeter data is not reported as it was computationally infeasible to run.}}
    \label{fig:muons_results}
\end{figure}

\noindent The scientific goal of this experiment is to quantify whether a high-granularity calorimeter would better constrain the energy of a muon (that is, lead to smaller confidence sets) than, for example, a detector that only measures the total energy of the incoming particle. 
To answer this question, we consider nested versions of the same energy measurement, where the inputs to our algorithms are of increasing dimensionality: (i) a 1D input which is equal to the sum over all the 
cells of the calorimeter (for each muon with deposited energy $E > 0.1$ GeV); (ii) 28 custom features extracted from the spatial and energy information of the calorimeter cells (see \cite{kieseler2022calorimetric}); and (iii) the full calorimeter measurement, $\mathbf{x} \in \mathbb{R}^{51{,}200}$. We then construct LF2I confidence sets for each data point using \texttt{BFF}. On the full calorimeter data, we learn the odds function through a convolutional neural network classifier derived from the regressor proposed in \cite{kieseler2022calorimetric}, and estimate critical values via quantile gradient boosted trees. For the 1D and 28D data sets, we instead learn odds through a gradient boosting classifier. In both cases, we use approximately $83\%$ of the data to learn the odds function ($B=738{,}930$) and $14\%$ to estimate critical values ($B^\prime=123{,}155$). For comparison, we also include results from SMC-ABC \citep{sisson2007sequential}, a popular LFI algorithm from the Approximate Bayesian Computation literature. To provide a fair assessment of the results, SMC-ABC uses all the simulations that LF2I exploits separately (i.e., $B+B^\prime = 862{,}085$). The remaining data points ($B^{\prime\prime}=24{,}631$) are used for validation and diagnostics of both methods.\\
\\
Figure~\ref{fig:muons_results} (center) shows that LF2I with the \BFF test statistic achieves the nominal level of coverage ($68.3\%$) regardless of the data set used. This is consistent with Theorem~\ref{thm:valid_tests}: as long as the quantile regression is well estimated, LF2I confidence sets are guaranteed to be valid at the nominal $(1-\alpha)$ level regardless of how well the test statistic is estimated. On the other hand, SMC-ABC is overly conservative with credible intervals that strongly over-cover across the whole parameter space. 
As to constraining power (interval length), Figure~\ref{fig:muons_results} (right) shows that SMC-ABC credible intervals are significantly wider than LF2I confidence sets for both the 1D and 28D data sets (running SMC-ABC on the $51{,}200$-dimensional full calorimeter data was computationally infeasible, and we were not able to report the results). Finally, note how the amount of information in the data directly influences the size of LF2I confidence sets: going from the 1D data set to the full calorimeter leads to noticeably smaller confidence intervals, and hence higher constraining power.}

\addEJS{
\paragraph{Remark on validity and computational cost} SMC-ABC does not have the right coverage, because the goal of ABC is to construct Bayesian credible regions and not valid confidence sets; see, e.g., \cite{hermans2021averting} for other examples of SMC-ABC under- or over-covering. Furthermore, note that (i) LF2I is amortized: once training is done, confidence sets can be efficiently computed on an arbitrary number of observations without having to retrain the algorithms; and (ii) there is no need for a prior dimension reduction of the data (that is, we can directly input the three-dimensional image). Specifically, LF2I required approximately 10 and 5 CPU minutes on an AMD's EPYC 7763 machine to train the odds classifier and the quantile regressor respectively, and less than a second to obtain confidence intervals all at once for all observations (in this example,  unique $24{,}631$ ``test'' muons) regardless of their dimensionality. In contrast,  SMC-ABC required approximately 1 CPU hour for \textit{each} observation even for the lower-dimensional 1D and 28D data sets.}

\section{Conclusions and Discussion}
\label{sec:conclusion}
\paragraph{Validity} Our proposed  LF2I methodology leads to frequentist confidence sets and hypothesis tests with finite-sample guarantees (when there are no nuisance parameters). {\em Any} existing or new test statistic -- that is, not only estimates of the LR or BF statistics -- can be plugged into our framework to create tests that control type I error. The implicit assumption is that the null distribution of the test statistic varies smoothly in parameter space. If that condition holds, then we can efficiently leverage quantile regression methods to construct valid confidence sets by a Neyman inversion of simple hypothesis tests, without having to rely on asymptotic results. 
    
 \paragraph{Nuisance parameters and diagnostics} For small sample sizes, no theorem guarantees whether profiling or marginalizing nuisance parameters will provide better frequentist coverage for the parameter of interest \citep[Section 12.5.1]{cousins2018lectures}. It is generally believed that hybrid resampling methods return approximately valid confidence sets, but that a rigorous check of validity is infeasible when the true solution is not known. Our diagnostic branch presents practical tools for assessing empirical coverage across the entire parameter space (including nuisance parameters). After seeing the results, one can decide which method is most appropriate for the application at hand. \addEJS{For example, in the Poisson counting experiment of Section~\ref{sec:hep_example}, LF2I diagnostics revealed that \texttt{h-BFF} (which averages the estimated odds over nuisance parameters) returned smaller confidence intervals, but at the cost of under-covering in some regions of the parameter space.}
 
\paragraph{Power} Statistical power is the hardest property to achieve in practice in LFI. This is the area where we foresee that most statistical and computational advances will take place. \addEJS{As shown theoretically in Theorem~\ref{thm:bound_average} and empirically in Supplementary Material~K, the power (or size) of LF2I confidence sets depends not only on the theoretical properties of the (exact) test statistics, but is also influenced by how precisely we are able to estimate it. In the case of \ACORE and \texttt{BFF}, the latter can be divided in} (i) how well we are able to estimate the likelihood or odds function (a statistical estimation error), and (ii) how accurate are the integration or maximization procedures we use (a purely numerical error); \addEJS{see Supplementary Material~H for a more precise breakdown of the sources of error in LF2I confidence sets, particularly for \ACORE and \texttt{BFF}.} Machine learning offers exciting possibilities on both fronts. For example, with regards to (i), \cite{Brehmer2020MiningGold} offers compelling evidence that one can can dramatically improve estimates of the likelihood $p(\x|\theta)$ for $\theta \in \Theta$, or the likelihood ratio $p(\x|\theta_1,\theta_2)$ for $\theta_1,\theta_2 \in \Theta$, by a ``mining gold'' approach that extracts additional information from the simulator about the latent process. Future work could incorporate such an approach into the LF2I framework, with the calibration and diagnostic branches as separate modules.

\paragraph{Other test statistics} Our work presents also another new direction for LF2I: So far frequentist LFI methods have been estimating either likelihoods or likelihood ratios, and then often relying on asymptotic properties of the LR statistic. We note that there are settings where it may be easier to either estimate the posterior $p(\theta|\x)$ rather than the likelihood $p(\x|\theta)$\addEJS{, or alternatively to obtain point estimates for parameters directly via predictions algorithms.}
Because the LF2I framework is agnostic to which
algorithms we use to construct the test statistic itself, we can potentially leverage methods that estimate the conditional mean $\mathbb{E}[\theta|\x]$ and variance $\mathbb{V}[\theta|\x]$ to construct frequentist confidence sets and hypothesis tests for $\theta$ with finite-sample guarantees. \addEJS{For example, \cite{masserano2023simulator} uses $T=\frac{(\mathbb{E}[\theta|\x]-\theta_0)^2}{\mathbb{V}[\theta|\x]}$, which in some scenarios corresponds to the Wald statistic for testing $H_{0,\theta_0}: \theta = \theta_0$ against $H_{1,\theta_0}: \theta \neq \theta_0$ \cite{wald1943tests}, as an attractive alternative to get LF2I confidence sets from prediction algorithms and posterior estimators.}\\
\\
\noindent\addEJS{See Appendices A-F for proofs and details on the algorithms, and refer to the separate Supplementary Material file\footnote{Available at \href{https://lucamasserano.github.io/data/LF2I\_supplementary\_material.pdf}{https://lucamasserano.github.io/data/LF2I\_supplementary\_material.pdf}.} for additional experiments and results referenced in the main text.}

\subsection{Related Work}\label{sec: related_work}

\paragraph{Classical statistical inference in high-energy physics (HEP)} LF2I is inspired by pioneering work in HEP that adopted classical hypothesis tests and Neyman confidence sets for the discovery of new physics \citep{Feldman1998UnifyingApproach, cowan2011formulaeInferenceHEP, Aad2012HiggsBoson, chatrchyan2012observation, Cranmer2015PraticalStatsLHC}. Our work grew from the discussion in HEP regarding theory and practice, and open problems such as how to efficiently construct Neyman confidence sets for general settings \citep{cowan2011formulaeInferenceHEP}, how to assess coverage across the parameter space without costly Monte Carlo simulations \citep{cousins2018lectures}, and how to choose hybrid techniques in practice \citep{cousins2006treatment}. \addEJS{This paper proposes a general approach to solve the above-mentioned open problems with a modular framework that can be adapted to fit the data at hand.}

\paragraph{Universal inference} Recently, \cite{wasserman2020universal} proposed a ``universal'' inference test statistic for constructing valid confidence sets and hypothesis tests with finite-sample guarantees without regularity conditions. The assumptions are that the likelihood $\mathcal{L}(\mathcal{D}; \theta)$ is known and that one can compute the maximum likelihood estimator (MLE). \addEJS{Our LF2I framework does {\em not} require a tractable likelihood, but it assumes that we have regression methods that can estimate the chosen test statistic and its critical values. In tractable likelihood settings where both universal inference and LF2I apply, the LF2I approach leads to more powerful tests than universal inference (see, e.g., Figure~11 in Supplementary Material).}

\addEJS{\paragraph{Simulation-based calibration of Bayesian posterior distributions}  In Bayesian inference, the posterior distribution $\pi(\theta|\x)$ is fundamental for quantifying uncertainty about the parameter $\theta$ given the data $\x$. Recent methods have been developed to assess the quality of estimated posterior distributions; that is, assessing whether an estimate  $\widehat \pi(\theta|\x)$ is consistent with the posterior distribution $\pi(\theta|\x)$ implied by the assumed prior and likelihood \citep{dey2021re,zhao2021diagnostics, dey2022calibrated, linhart2023c2st, lemos2023sampling}. \textit{The calibration in LF2I is fundamentally different:} Even if posteriors are calibrated in the sense that $\widehat \pi(\theta|\x)=\pi(\theta|\x)$ for every $\x$ and $\theta$, confidence sets derived from it will not necessarily have the correct empirical coverage (according to Eq.\ref{eq:cond_coverage}). LF2I is agnostic to the choice of the test statistic (for instance, whether the test statistic is formed from likelihoods or posteriors \citep{masserano2023simulator}), and provides guarantees of how well we are able to constrain the true parameters of interest regardless of the choice of the prior or proposal distribution $\pi(\theta)$.}

\paragraph{Likelihood-free inference via machine learning}  Recent LFI methods have been using simulators output as training data to learn surrogate models for inference; see \cite{Cranmer2020Review} for a review. These techniques use synthetic data simulated across the parameter space to directly estimate key quantities, such as:
\begin{enumerate}
\item {\em posteriors} $p(\theta | \x)$ \citep{blum2010non, marin2016abcrf, Papamakarios2016SNP,  lueckamnn2017Posterior, greenberg2019posterior, chen2019gaussiancopula, izbicki2019abc, radev2020Bayesflow};
\item {\em likelihoods} $p(\x | \theta)$ \citep{Wood2010GaussSynth, meeds2014gps, wilkinson2014accelerating, gutmann2016bolfi, fasiolo2018bayesiansynth, lueckmann2019likelihood, papamakarios2019likelihood, picchini2020adaptivebayesiansynth, jarvenpaa2021bolfi}; or   
\item {\em density ratios}, such as the likelihood-to-marginal ratio $p(\x | \theta)/p(\x)$ \citep{izbicki2014high, thomas2021lfire, hermans2020likelihoodfree, durkan2020constrastive},\footnote{In 2014, Izbicki et al. approximate likelihoods for high-dimensional data (such as 2D images) via density ratios \citep[Equation 3]{izbicki2014high} and kernel methods, building on Izbicki's PhD thesis work on spectral series approaches to high-dimensional nonparametric inference. The kernel approximate likelihood approach was later superseded by neural SBI approaches.} the likelihood ratio $p(\x | \theta_1)/p(\x|\theta_2)$ for $\theta_1, \theta_2 \in \Theta$ \citep{Cranmer2015LikRatio, Brehmer2020MiningGold} or the profile-likelihood ratio \citep{heinrich2022learning}.\footnote{ \ACORE and \BFF are based on estimating the odds $\mathbb{O}(\X;\theta)$ at $\theta \in \Theta$ (Equation~\ref{eq:odds_def}); this is a ``likelihood-to-marginal ratio'' approach, which estimates a one-parameter function as in the original paper by~\cite{izbicki2014high}. The likelihood ratio $\mathbb{OR}(\X;\theta_0,\theta_1)$ at $\theta_0, \theta_1 \in \Theta$ (Equation~\ref{eq:oddsratio_def}) is then computed from the odds function, without the need for an extra estimation step.}
\end{enumerate}
\revEJS{Recently, there have also been works that directly predict parameters $\theta$ of intractable models using neural networks \cite{gerber2021fast, lenzi2021neural} (that is, they do not estimate posteriors, likelihoods or density ratios).} \revEJS{In addition, new methods such as normalizing flows \citep{Papamakarios2019FlowsReview} and other neural density estimators are revolutionizing LFI in terms of sample efficiency and capacity, and will continue to do so.}\\

\noindent Nonetheless, although the goal of LFI is inference on the unknown parameters $\theta$, it remains an open question whether a given LFI algorithm produces reliable measures of uncertainty\addTwo{, as current methods lack guarantees of local (instance-wise) validity and power for a finite number of observations. They also have no practical diagnostics to assess local coverage across the parameter space. Our framework can be used in combination with any LFI approach that relies on a test statistic (such as the LRT) to provide both local coverage and diagnostics. Finally, thanks to the modular structure of LF2I, the diagnostic branch can be used separately to evaluate whether other approaches (like ABC and posterior methods that return credible regions) have good frequentist coverage, and in cases where they do not, LF2I can  identify regions of the parameter space of over- or under-confidence.}

\begin{appendix}

\section{Estimating Odds}\label{app:methods}

\begin{algorithm}[b!]
    \small
    \caption{Generate a labeled sample of size $B$ for estimating odds}
    \label{alg:joint_y}
    \begin{flushleft}
        {\bf Input:} simulator $F_\theta$; reference distribution $G$; proposal distribution $\pi_{\Theta}$ over parameter space; number of simulations $B$; parameter $p$ of Bernoulli distribution
        {\bf Output:} labeled training sample $\T$
    \end{flushleft}
    \begin{algorithmic}[1]
        \State Set $\T \gets \emptyset$
        \For{$i$ in $\{1,...,B\}$}
        \State Draw parameter value $\theta_i \sim \pi_{\Theta}$
        \State Draw $Y_i \sim {\rm Ber}(p)$
        \If{$Y_i==1$}
            \State  Draw sample $\X_i \sim  F_{\theta_i}$ 
        \Else
            \State Draw sample $\X_i \sim G$
        \EndIf
        \State $\T \gets \T \cup \left(\theta_i,\X_i,Y_i\right)$
        \EndFor
        \State \textbf{return} $\T=\{\theta_i,\X_i,Y_i\}_{i=1}^B$  
    \end{algorithmic}
\end{algorithm}

Algorithm~\ref{alg:joint_y} shows how to create the training set $\T$ for estimating odds. Out of the total 
 \revJMLR{number of simulations}
$B$, a proportion $p$ is generated by the stochastic forward simulator $F_\theta$ at different parameter values $\theta$, while the rest is sampled from a reference distribution $G$. Note that $G$ can be any distribution that dominates $F_\theta$. If $G$ is the marginal distribution $F_\x$ and $n=1$, then computations for \texttt{BFF} are simplified because its denominator equals one. \addEJS{Algorithm~\ref{alg:sample_from_marginal} shows how to sample from the marginal distribution $F_\x$. In practice, if the data is pre-simulated, one can sample from the (empirical) marginal using permutations to break the relationship between $\theta$ and $\X$ for $\X \sim G=F_\x$.}

\begin{algorithm}[t!]
    \small
    \caption{Sample from the marginal distribution $G=F_{\X}$}
    \label{alg:sample_from_marginal}
    \begin{flushleft}
        {\bf Input:} simulator $F_\theta$; proposal distribution $\pi_{\Theta}$ over parameter space\\
        {\bf Output:} sample $\X_i$ from the marginal distribution $F_{\X}$
    \end{flushleft}
    \begin{algorithmic}[1]
        \State Draw parameter value $\theta_i \sim \pi_{\Theta}$
        \State Draw sample $\X_i \sim F_{\theta_i}$
        \State \textbf{return} $\X_i$
    \end{algorithmic}
\end{algorithm}

\section{Estimating p-values}\label{app:p_values}
Given observed data $D$ and a test statistic $\lambda$, we can compute p-values $p(D;\theta_0):={\mathbb P}_{\mathcal{D} | \theta _{0}} \left (
\lambda ({\mathcal D};\theta _{0})< \lambda (D;\theta _{0})\right )$ for each hypothesis $H_{0,\theta _{0}}:\theta =\theta _{0}$. Algorithm~\ref{alg:estimate_pvalues} describes how to estimate such p-values for all $\theta _{0} \in \Theta$ simultaneously.

\begin{algorithm}[h!]
    \small
    \caption{Estimate p-values $p(D;\theta_0)$ given observed data $D$ for a level-$\alpha$ test of  $H_{0, \theta_0}: \theta = \theta_0$ vs. $H_{1, \theta_0}: \theta \neq \theta_0$, for all $\theta_0 \in \Theta$ simultaneously.}
    \label{alg:estimate_pvalues}
    \begin{flushleft}
        {\bf Input:} observed data $D$; simulator $F_\theta$; number of simulations $B'$; $\pi_\Theta$ (fixed proposal distribution over the parameter space $\Theta$); test statistic $ \lambda$; probabilistic classifier
        {\bf Output:} estimated p-value  $\widehat{p}(D;\theta)$ for all $\theta=\theta_0 \in \Theta$
    \end{flushleft}
    \begin{algorithmic}[1]
        \State{ Set $\T' \gets \emptyset$}
        \For{i in $\{1,\ldots,B'\}$} 
        \State  Draw parameter $\theta_i \sim \pi_{\Theta}$
        \State Draw sample $\X_{i,1},\ldots,\X_{i,n}  \stackrel{iid}{\sim}  F_{\theta_i}$
        \State Compute test statistic $ \lambda_i \gets  \lambda((\X_{i,1},\ldots,\X_{i,n});\theta_i) $
        \State Compute indicator $Z_i \gets \I\left(\lambda_i <  \lambda(D;\theta_i) \right) $
        \State  $\T' \gets \T' \cup  \{(\theta_i,Z_{i})\}$
        \EndFor
        \State  Use $\T'$ to learn the \revJMLR{p-value} function $\widehat{p}(D;\theta)$ using $Z$ as the label for each $\theta$\\
        \textbf{return} $\widehat{p}(D;\theta_0)$
    \end{algorithmic}
\end{algorithm}

\section{Constructing Confidence Sets}\label{app:confidence_sets}

Algorithm~\ref{alg:conf_reg} details the  construction of LF2I confidence sets with \ACORE and \BFF as defined in Section~\ref{sec:LF2I} (the algorithm based on p-value estimation is analogous). Algorithm~\ref{alg:conf_reg_nuisance} details the construction of the (hybrid) \ACORE and \BFF confidence sets defined in Section~\ref{sec:nuisance} for the general setting with nuisance parameters. Note that the first chunk on estimating the odds and the last chunk with Neyman inversion are the same for \ACORE and \texttt{BFF}. Furthermore, the test statistics are the same whether or not there are nuisance parameters. 

\begin{algorithm}[t!]
    \small
    \caption{Construct $(1-\alpha)$ confidence set for $\theta$ (no nuisance parameters)}
    \label{alg:conf_reg}
    \begin{flushleft}
        {\bf Input:} simulator $F_\theta$; proposal distribution $\pi$ over $\Theta$; parameter $p$ of Bernoulli; number of simulations $B$ (test statistic); number of simulations $B'$ (critical values); probabilistic classifier; observations $D=\{\x_i^{\text{obs}}\}_{i=1}^n$; level $\alpha \in (0,1)$; size of evaluation grid over parameter space $n_{\rm grid}$; test statistic $\lambda$ (\ACORE or \texttt{BFF})\\
        {\bf Output:} $\theta$ evaluation points in confidence set $\widehat{R}(D)$
    \end{flushleft}
    \begin{algorithmic}[1]
        \State \codecomment{Estimate odds}
	\State Generate labeled sample $\T$ according to Algorithm~\ref{alg:joint_y}
	\State Apply probabilistic classifier to $\T$ to learn $\widehat{\P}(Y=1|\theta,\X),$ for all \revJMLR{$\theta \in \Theta$} and $\X \in \mathcal{X}$	
        \State Let the estimated odds $\widehat{\mathbb{O}}(\X;\theta) \gets \frac{\widehat{\P}(Y=1|\theta,\X)}{\widehat{\P}(Y=0|\theta,\X)}$ \\
        \State \codecomment{Compute \revJMLR{cut-offs for \ACORE or \BFF}}
        \If{$\lambda$ == \texttt{ACORE}}
            \State Let $\lambda(\D; \theta) \gets \widehat{\Lambda}(\D; \theta)$ be the \ACORE statistic (Equation \ref{eq:ACORE_statistic}) with estimated odds
        \ElsIf{\texttt{test\_stat} == \texttt{BFF}}
            \State Let $\lambda(\D; \theta) \gets \widehat{\tau}(\D; \theta)$ be the \BFF statistic (Equation \ref{eq:BFF_statistic}) with estimated odds
        \EndIf
        \State{Learn critical values $\hat C_{\theta}$ according to Algorithm~\ref{alg:estimate_cutoffs}}\\
        
        \State \codecomment{Confidence sets for $\theta$ via Neyman inversion}
        \State Initialize confidence set $\widehat{R}(D) \gets \emptyset$
        \State Let ${\rm L}_\Theta$ be a lattice over $\Theta$ with $n_{\rm grid}$ elements
        \For{$\theta_0 \in {\rm L}_\Theta$}
    	\If{ $\lambda(D; \theta_0) \geq \widehat{C}_{\theta_0}$ }
                \State $\widehat{R}(D) \gets  \widehat{R}(D)  \cup \{ \theta_0 \}$
            \EndIf
        \EndFor
	\State \textbf{return} confidence set  $\widehat{R}(D)$
    \end{algorithmic}
\end{algorithm}

\begin{algorithm}[t!]
    \small
    \caption{Construct confidence set for $\phi$ with (approximate) coverage $1-\alpha$ under the presence of nuisance parameters}
    \label{alg:conf_reg_nuisance}
    \begin{flushleft}
        {\bf Input:} simulator $F_\theta$; proposal distribution $\pi$ over $\Theta = \Phi\times \Psi$; parameter $p$ of Bernoulli; number of simulations $B$ (test statistic); number of simulations $B'$ (critical values); probabilistic classifier; observations $D=\{\x_i^{\text{obs}}\}_{i=1}^n$; level $\alpha \in (0,1)$; size of evaluation grid over parameter space, $n_{\rm grid}$; test statistic $\lambda$ (\ACORE or \texttt{BFF})\\
        {\bf Output:} $\phi$ evaluation points in confidence set $\widehat{R}(D)$
    \end{flushleft}
    \begin{algorithmic}[1]
        \State \codecomment{Estimate odds}
	\State Generate labeled sample $\T$ according to Algorithm~\ref{alg:joint_y}
        \State Apply probabilistic classifier to $\T$ to learn $\widehat{\P}(Y=1|\theta,\X),$ $\forall \text{ }\theta = (\phi, \psi ) \in \Theta, \X \in \mathcal{X}$	
        \State Let the estimated odds $\widehat{\mathbb{O}}(\X;\theta) \gets \frac{\widehat{\P}(Y=1|\theta,\X)}{\widehat{\P}(Y=0|\theta,\X)}$\\
     
        \State \codecomment{Compute (hybrid) critical values for \texttt{h-ACORE} or \texttt{h-BFF}}
        \If{$\lambda$ == \texttt{ACORE}}
            \State Let $\widehat\psi_\phi \gets \arg\max_{\psi \in \Psi}  \prod_{i=1}^n \widehat{\mathbb{O}}(\x_i^{\rm obs};(\phi,\psi))$ for every $\phi$
            \State Let $\lambda(\D; \phi) \gets \widehat{\Lambda}\left(\D; (\phi, \widehat\psi_\phi)\right)$ be \ACORE (Equation \ref{eq:ACORE_statistic}) with estimated odds
            \State Generate $\mathcal{T}^{\prime}$ as in Algorithm \ref{alg:estimate_cutoffs}
         using the proposal $\pi'((\phi,\psi)) \propto \pi(\phi)\times  \delta_{\hat{\psi}_{\phi}}(\psi)$
            \State Learn $\hat C_{\phi}= \widehat{F}^{-1}_{\lambda \left(\D;\phi \right) \big| \left(\phi,\widehat \psi_{\phi} \right)}(\alpha)$ for every $\phi$ as in Algorithm~\ref{alg:estimate_cutoffs} using $\T^\prime$
        \ElsIf{$\lambda$ == \texttt{BFF}}
            \State Let $\pi_\Psi(\psi)$ be the restriction of proposal distribution $\pi$ over $\Psi$
            \State Let $\lambda(\D; \phi) \gets \widehat{\tau}(\D; \phi)$
            be the \BFF statistic (Equation \ref{eq:BFF_statistic}) with estimated odds
            \State Learn $\hat C_{\phi}= \widehat{F}^{-1}_{\lambda \left(\D;\phi \right) \big| \left(\phi\right)}(\alpha)$ for every $\phi$ (no $\psi$) as in Algorithm~\ref{alg:estimate_cutoffs}
        \EndIf\\
        
        \State \codecomment{Confidence sets for $\phi$ via Neyman inversion}
        \State Initialize confidence set $\widehat{R}(D) \gets \emptyset$
        \State Let ${\rm L}_\Phi$ be a lattice over $\Phi$ with $n_{\rm grid}$ elements
        \For{$\phi_0 \in {\rm L}_\Phi$} 
    	\If{ $\lambda(D; \phi_0) \geq \widehat{C}_{\phi_0}$ }
                \State $\widehat{R}(D) \gets  \widehat{R}(D)  \cup \{ \phi_0 \}$
            \EndIf
        \EndFor
	\State \textbf{return} confidence set  $\widehat{R}(D)$
    \end{algorithmic}
\end{algorithm}

\section{Theoretical Guarantees of Power for \ACORE with Calibrated Critical Values}\label{app:acore}
  
Next, we show\revJMLR{, for finite $\Theta$,} that as long as the probabilistic classifier is consistent and the critical values are well estimated (which holds for large $B'$ according to Theorem~\ref{thm:convergenceCutoffs}), the power of the \texttt{ACORE} test converges to the power of the LRT as $B$ grows.  
 
\begin{thm}
\label{thm:convergenceLRT}
For each $C\in \mathbb{R}$, let $\widehat{\phi}_{B,C}(\mathcal{D})$ be the test based on
the \ACORE statistic $\hat \Lambda_B$  \revJMLR{with critical value $C$}\footnote{That is, $\widehat{\phi}_{B,C}(\mathcal{D})=1  \iff \hat \Lambda_B(\D; \Theta_0) < C$.} \revJMLR{for number of simulations $B$ in Algorithm~\ref{alg:joint_y}.} Moreover, 
let 
$\phi_{C}(\mathcal{D})$ be the likelihood ratio test with  critical value
$C$. 
If, for every
$\theta \in \Theta$, the probabilistic classifier is such that
$$\widehat{\P}(Y=1|\theta,\X) \xrightarrow[B \longrightarrow\infty]{\enskip P \enskip} \P(Y=1|\theta,\X),$$
where $|\Theta|<\infty$, and $\widehat{C}_B$ 
is chosen  such that
$\widehat{C}_B
\xrightarrow[B \longrightarrow\infty]{\enskip D \enskip}C$ for a given $C \in \mathbb{R}$,
then, for every $\theta \in \Theta$,
$$\P_{\mathcal{D},\T|\theta}\left(\widehat{\phi}_{B,\widehat{C}_B}(\mathcal{D})=1 \right) 
\xrightarrow[B \longrightarrow\infty]{} \P_{\mathcal{D}|\theta}\left(\phi_{C}(\mathcal{D})=1\right).
$$
\end{thm}

\begin{proof}
Because $\widehat{\P}(Y=1|\theta,\X) \xrightarrow[B \longrightarrow\infty]{\enskip P \enskip} \P(Y=1|\theta,\X)$,  it follows directly from the properties of convergence in probability that
for every $\theta_0,\theta_1 \in \Theta$
 \begin{align*}
    \sum_{i=1}^n  \log & \left( \widehat{\Or}(\X_i^{\text{obs}};\theta_0,\theta_1)\right) 
    \xrightarrow[B \longrightarrow\infty]{\enskip P \enskip}
 \sum_{i=1}^n  \log \left( \Or(\X_i^{\text{obs}};\theta_0,\theta_1)\right). 
 \end{align*}
 The continuous mapping theorem implies that
 $$\hat \Lambda_B(\D; \Theta_0) \xrightarrow[B \longrightarrow\infty]{\enskip P \enskip}
 \sup_{\theta_0 \in \Theta_0}\inf_{\theta_1 \in \Theta} \sum_{i=1}^n  \log \left( \Or(\X_i^{\text{obs}};\theta_0,\theta_1)\right),$$
and therefore
$\hat \Lambda_B(\D; \Theta_0)$ converges in distribution to 
$\sup_{\theta_0 \in \Theta_0}\inf_{\theta_1 \in \Theta}\sum_{i=1}^n \linebreak   \log \left( \Or(\X_i^{\text{obs}};\theta_0,\theta_1)\right)$. 
Now, from Slutsky's theorem,
\begin{align*}
\hat \Lambda_B(\D; \Theta_0)-\widehat{C}_B
\xrightarrow[B \longrightarrow\infty]{\enskip  \enskip}\sup_{\theta_0 \in \Theta_0}\inf_{\theta_1 \in \Theta}
\sum_{i=1}^n  \log \left( \Or(\X_i^{\text{obs}};\theta_0,\theta_1)\right)-C. 
\end{align*} 
It follows that
\begin{align*}
 &\P_{\mathcal{D},\T|\theta}\left(\widehat{\phi}_{B,\widehat{C}_B}(\mathcal{D})=1 \right)=
 \P_{\mathcal{D},\T|\theta}\left(\hat \Lambda_B(\D; \Theta_0)-\widehat{C}_B\leq 0 \right) \\
 &\xrightarrow[B \longrightarrow\infty]{}
 \P_{\mathcal{D}|\theta}\Big(\sup_{\theta_0 \in \Theta_0}\inf_{\theta_1 \in \Theta} \sum_{i=1}^n  \log \left( \Or(\X_i^{\text{obs}};\theta_0,\theta_1)\right)-C \leq 0 \Big)\\
 &=\P_{\mathcal{D}|\theta}\left(\phi_{C}(\mathcal{D})=1 \right),
\end{align*}
where the last equality follows from Proposition~\ref{prop::consistency}.
\end{proof}
 
\section{Analysis of Critical Values for Experiments~\ref{sec:GMM} and \ref{sec:hep_example}}\label{app:critical_vals}

\begin{figure}[t!]
    \centering
    \includegraphics[width=0.9\textwidth]{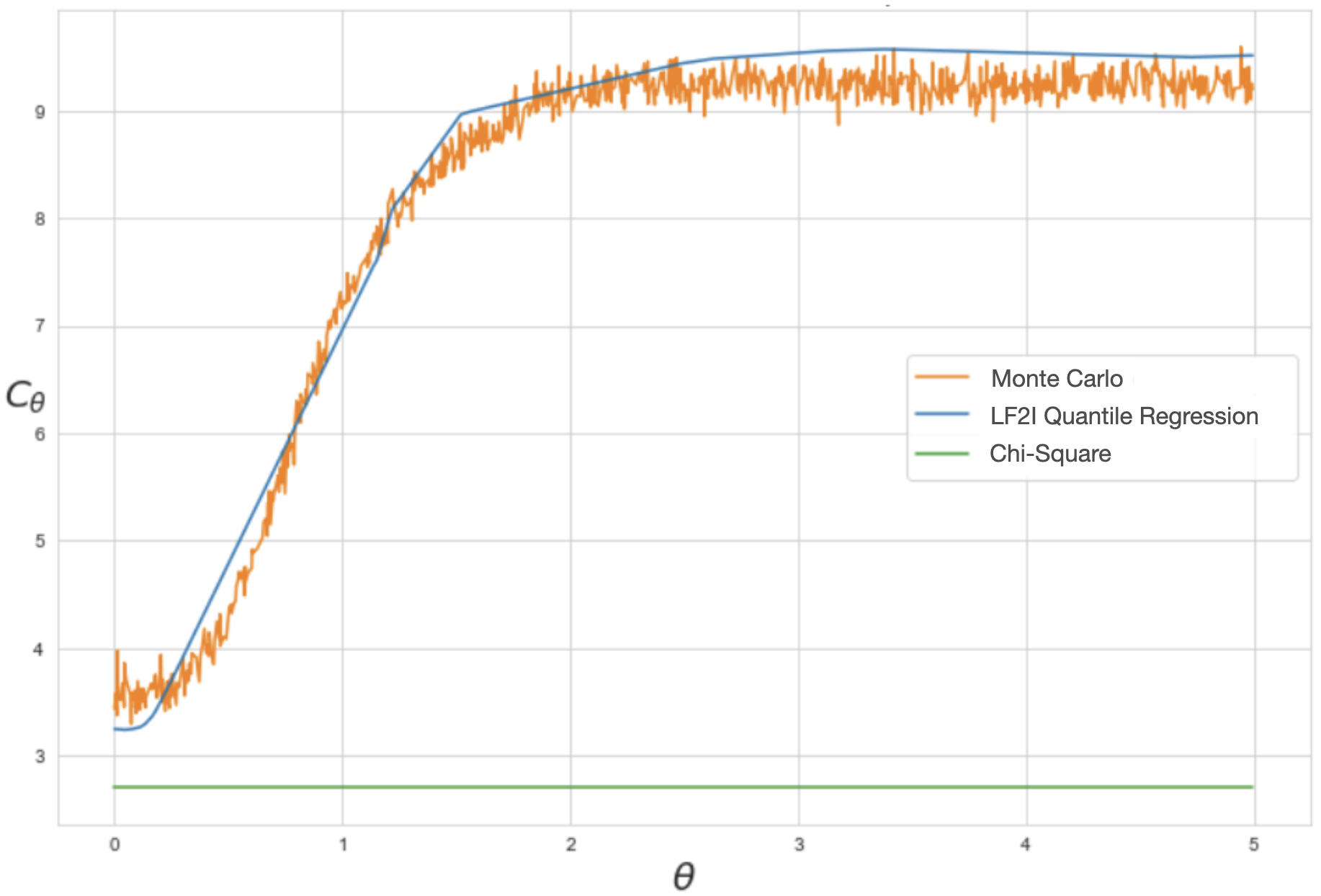}
    \caption{\small \addTwo{{\bf Comparison of critical values} obtained via  \textcolor{orange}{Monte Carlo}, the \textcolor{ForestGreen}{Chi-Square} asymptotic assumption of Wilks' Theorem, and  \textcolor{NavyBlue}{LF2I Quantile Regression}, for the GMM example of Section~\ref{sec:GMM}. }}\vspace{-5mm}
    \label{fig:critical_61}
\end{figure}

\begin{figure}[b!]
    \centering
    \includegraphics[width=1\textwidth]{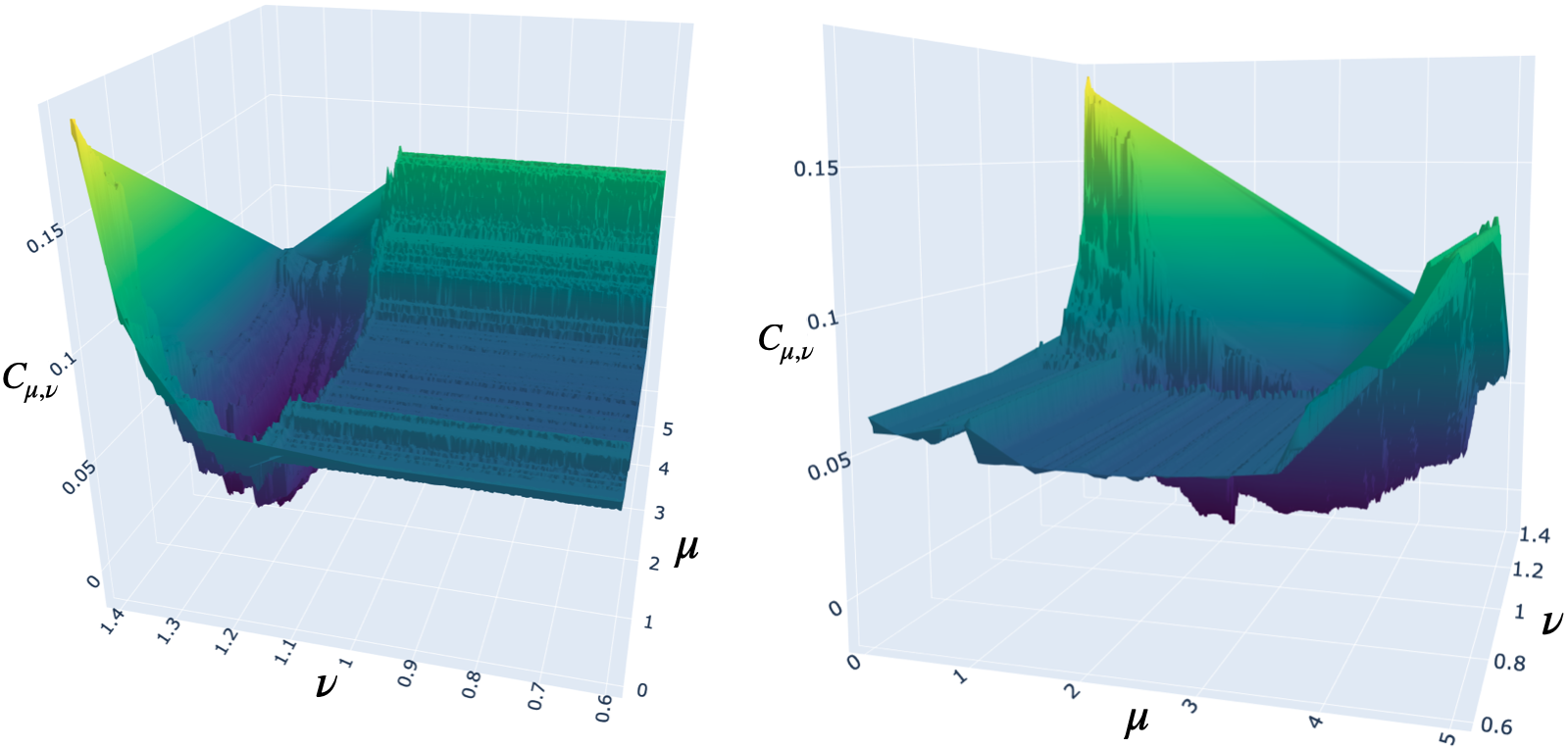}
    \caption{\small \addTwo{{\bf Critical values of $\texttt{h-ACORE}$ estimated via quantile regression} as a function of the parameter of interest $\mu$ and the nuisance parameter $\nu$, for the example of Section~\ref{sec:hep_example}. The figures show the same 2D surface from two different angles.}}
    \label{fig:critical62}
\end{figure}

\addTwo{
In this section we visualize how critical values vary across the parameter space $\Theta$ for the experiments of Sections~\ref{sec:GMM} and~\ref{sec:hep_example}. Figure~\ref{fig:critical_61} compares critical values for the exact LRT of the Gaussian Mixture Model (GMM) example, where the distribution of the test statistic is unknown, using three different methods: \\\\
\textit{i)} The first approach is to compute cutoffs via Monte Carlo (MC) simulations at fixed values of $\theta$. These critical values can be considered the “ground truth”, since for this one-dimensional example we were able to use a high-resolution grid and large batches at each grid point. Unfortunately, MC quickly becomes infeasible if the dimensionality of the parameter space increases. In addition, a scientist cannot adopt MC samples in practical settings, where one only has access to a pre-determined data set and not to the simulator itself. \\\\
\textit{ii)} The second approach is to assume that the cutoff is (asymptotically) constant across the parameter space. Here we have computed cutoffs assuming that Wilks' theorem holds and that the limiting distribution is a $\chi^2$-distribution, which is not the case. Indeed, the bottom central panel of Figure~\ref{fig:GMM_coverage} shows that the $\chi^2$-approximation achieves correct coverage only when $\theta = 0$ (i.e., when the GMM collapses to one Gaussian).\\\\
\textit{iii)} The third approach is to compute the critical values of the (known) test statistic via quantile regression (QR). With a very small calibration set ($0.1\%$ of the total simulations used for the MC approach), QR is able to approximate the quantile surface and achieve nominal coverage for all values of $\theta$ (see Figure~\ref{fig:GMM_coverage}).\\\\

\noindent Figure~\ref{fig:critical62} shows similar results  for the HEP example of Section~\ref{sec:hep_example}; here we visualize the  the critical values of $\texttt{h-ACORE}$ (estimated via LF2I) as a function of the parameter of interest $\mu$ and the nuisance parameter $\nu$. Again, we see evidence that the quantile surface is far from being constant, and that the test statistic is not pivotal. Hence, there is a need for a quantile regression that adapts to the varying distribution of the test statistic.
}

\section{Additional Proofs}
\label{app:theory}

\begin{proof}[Proof of Proposition \ref{prop::consistency}]
\revJMLR{\revEJS{Because $\nu$ dominates $F_\theta$, $G$ also dominates $F_\theta$.}
Let  $f(\x|\theta)$ be the density of $F_\theta$ with respect to $G$. } 
By Bayes rule, 
\revJMLR{$${\mathbb{O}}(\x;\theta):=\frac{\P(Y=1|\theta,\x)}{\P(Y=0|\theta,\x)}=\frac{f(\x|\theta)p}{(1-p)}.$$}
If $\ \widehat{\P}(Y=1|\theta,\x)=\P(Y=1|\theta,\x)$, then $\widehat{\mathbb{O}}(\x;\theta_0) = \mathbb{O}(\x;\theta_0)$. Therefore, 
\revJMLR{\begin{align*}
\widehat{\tau}(\D;\Theta_0) :=& \frac{\int_{\Theta_0}  \prod_{i=1}^n  \hat{\mathbb{O}}(\X_i^{\text{obs}};\theta)d\pi_0(\theta)}{ \int_{\Theta_1}  \prod_{i=1}^n  \hat{\mathbb{O}}(\X_i^{\text{obs}};\theta)  d\pi_1(\theta)}\\ 
=& \frac{\int_{\Theta_0}  \prod_{i=1}^n   \mathbb{O}(\X_i^{\text{obs}};\theta)d\pi_0(\theta)}{ \int_{\Theta_1}  \prod_{i=1}^n  \mathbb{O}(\X_i^{\text{obs}};\theta)  d\pi_1(\theta)} \\
=&   \frac{\int_{\Theta_0}  \prod_{i=1}^n  \frac{f(\X_i^{\text{obs}} | \theta)p}{(1-p)}d\pi_0(\theta)}{ \int_{\Theta_1}  \prod_{i=1}^n   \frac{f(\X_i^{\text{obs}} | \theta)p}{(1-p)}   d\pi_1(\theta)}\\
=&   \frac{\int_{\Theta_0}  \prod_{i=1}^n  f(\X_i^{\text{obs}} | \theta)d\pi_0(\theta)}{ \int_{\Theta_1}  \prod_{i=1}^n   f(\X_i^{\text{obs}} | \theta) d\pi_1(\theta)}
\end{align*}}
\revEJS{Moreover, the chain rule implies that $f(\x|\theta)=p(\x|\theta)h(\x)$, where $h(\x):=\frac{d\nu}{dG}(\x)$. It follows that
\begin{align*}
   \widehat{\tau}(\D;\Theta_0) =& \frac{\int_{\Theta_0}  \prod_{i=1}^n  f(\X_i^{\text{obs}} | \theta)d\pi_0(\theta)}{ \int_{\Theta_1}  \prod_{i=1}^n   f(\X_i^{\text{obs}} | \theta) d\pi_1(\theta)} \\ 
   &=\frac{\int_{\Theta_0}  \prod_{i=1}^n  p(\X_i^{\text{obs}} | \theta)h(\X_i^{\text{obs}})d\pi_0(\theta)}{ \int_{\Theta_1}  \prod_{i=1}^n   p(\X_i^{\text{obs}} | \theta) h(\X_i^{\text{obs}})d\pi_1(\theta)} \\ 
    &=\frac{\int_{\Theta_0}  \prod_{i=1}^n  p(\X_i^{\text{obs}} | \theta)d\pi_0(\theta)}{ \int_{\Theta_1}  \prod_{i=1}^n   p(\X_i^{\text{obs}} | \theta) d\pi_1(\theta)} \\
   &=\frac{\int_{\Theta_0}\mathcal{L}(\D;\theta)d\pi_0(\theta)}{\int_{\Theta_1}\mathcal{L}(\D;\theta) d\pi_1(\theta)}\\
&= \text{BF}(\D;\Theta_0).
\end{align*} }
\end{proof}

\revEJS{\begin{proof}[Proof of Theorem \ref{thm:valid_tests}] By definition, for every fixed $c_{B'}$,  
$\P_{\mathcal{D}|\theta_0,C_{B'}}(\lambda(\mathcal{D};\theta_0) \leq c_{B'}) = F(c_{B'}|\theta_0)$. It follows that the random variable $\P_{\mathcal{D}|\theta_0,C_{B'}}(\lambda(\mathcal{D};\theta_0) \leq C_{B'}) = F(C_{B'}|\theta_0)$. Moreover,
 by construction,
$\alpha=\widehat F_{B'}(C_{B'}|\theta_0)$. It follows that
\begin{align*}
     |\P_{\mathcal{D}|\theta_0,C_{B'}}(\lambda(\mathcal{D};\theta_0) \leq C_{B'}) - \alpha| &=  | F(C_{B'}|\theta_0) - \alpha|    \\
     &= |F(C_{B'}|\theta_0)  - \widehat F_{B'}(C_{B'}|\theta_0)|\\
     &\leq \sup_{\lambda \in \mathbb{R}}|  F(\lambda|\theta_0)-\widehat F_{B'}(\lambda|\theta_0)|\xrightarrow[B' \longrightarrow\infty]{\enskip P \enskip} 0.
\end{align*}
The result follows from the fact that convergence in probability to a constant implies almost sure convergence.
\end{proof}}

\revJMLR{
\begin{proof}[Proof of Theorem \ref{thm:valid_tests_rate}]
The proof  follows from applying the convergence rate to the last equation in the proof of Theorem \ref{thm:valid_tests}.
\end{proof}
}

\revJMLR{ \begin{Assumption}[Uniform consistency \revEJS{in $\theta$ and $\lambda$}]
\label{assum:quantile_consistent}
Let $\hat F_{B'}(\cdot|\theta)$ be the estimated cumulative distribution function of the test statistic $\lambda(\mathcal{D};\Theta_0)$ conditional on $\theta$ based on a sample $\mathcal{T}^{'}$ with size $B'$ implied by the quantile regression, and
let $ F(\cdot|\theta)$ be its true \addTwo{distribution given $\theta$}.
Assume that the quantile regression estimator is such that
$$\sup_{\theta \in \Theta_0, \lambda \in \mathbb{R}}|\hat F_{B'}(\lambda|\theta)-  F(\lambda|\theta)|\xrightarrow[B' \longrightarrow\infty]{\enskip P \enskip} 0.$$
This assumption holds, for instance, for quantile regression forests \citep{meinshausen2006qrforest} under additional assumptions  (see Proposition~\ref{prop:uniform_convergence}).
\end{Assumption}}

\revJMLR{
\begin{prop}
\label{prop:uniform_convergence}
If, for every $\theta \in \Theta_0$, 
the quantile regression estimator is such that
\begin{align}
    \label{eq:point_conv}
\sup_{\lambda \in \mathbb{R}}|\hat F_{B'}(\lambda|\theta)-  F(\lambda|\theta)|\xrightarrow[B' \longrightarrow\infty]{\enskip P \enskip} 0
\end{align}
and either
\begin{itemize}
 \item $|\Theta|<\infty$ or,
    \item $\Theta$ is a compact subset of $\mathbb{R}^d$, and the function $g_{B'}(\theta) = \sup_{t \in \mathbb{R}} | \hat F_{B'}(t|\theta) - F(t|\theta)|$ is almost surely
continuous in $\theta$ and strictly decreasing in $B'$,
\end{itemize}
\end{prop}
then Assumption~\ref{assum:quantile_consistent} holds.} 
\vspace{4mm}

\revJMLR{
\begin{proof}
If $|\Theta|<\infty$,
 the union bound and Equation \ref{eq:point_conv} imply that
 \begin{align}
 \label{eq:sup0}
 \sup_{\theta \in \Theta_0} \sup_{\lambda \in \mathbb{R}}|\hat F_{B'}(\lambda|\theta)- F(\lambda|\theta)|\xrightarrow[B' \longrightarrow\infty]{\enskip P \enskip} 0.    
 \end{align}
  Similarly, by Dini's theorem, Equation~\ref{eq:sup0} also holds if $\Theta$ is a compact subset of $\mathbb{R}^d$, and the function $g_{B'}(\theta)$ is
continuous in $\theta$ and strictly decreasing in $B'$.
\end{proof}
}

\revJMLR{
\begin{thm}
 \label{thm:convergenceCutoffs}
Let 
$C_{B'} \in \mathbb{R}$ be the 
critical value 
of the test based on
a \revEJS{absolutely} continuous statistic  $\lambda(\mathcal{D};\Theta_0)$ chosen according to Algorithm~\ref{alg:estimate_cutoffs}
for a fixed $\alpha \in (0,1)$. If the quantile
estimator satisfies Assumption~\ref{assum:quantile_consistent},
then 
$$ C_{B'}  \xrightarrow[B' \longrightarrow\infty]{\enskip P \enskip} C^*,$$
where $C^*$ is such that 
$$\sup_{\theta \in \Theta_0}\P_{\mathcal{D}|\theta}(\lambda(\mathcal{D};\Theta_0) \leq C^*) = \alpha.$$
\end{thm}}
 \vspace{4mm}
 
\revJMLR{
\begin{proof}
Assumption~\ref{assum:quantile_consistent} implies that 
 $$\sup_{\theta \in \Theta_0} |\hat F^{-1}_{B'}(\alpha|\theta)- F^{-1}(\alpha|\theta)|\xrightarrow[B' \longrightarrow\infty]{\enskip P \enskip} 0.$$
 The result then follows from the fact that
 \begin{align*}
 0 \leq    |C_{B'}-C^*|&= |\sup_{\theta \in \Theta_0} \hat F^{-1}_{B'}(\alpha|\theta)-\sup_{\theta \in \Theta_0}  F^{-1}(\alpha|\theta)| \\
     &\leq \sup_{\theta \in \Theta_0} |\hat F^{-1}_{B'}(\alpha|\theta)- F^{-1}(\alpha|\theta)|,
 \end{align*}
 and thus 
 $$|C_{B'}-C^*| \xrightarrow[B' \longrightarrow\infty]{\enskip P \enskip} 0.$$
 \end{proof}
  }
  
\revJMLR{
\begin{Lemma}
\label{lemma:convergence_dist}
Let $g_1,g_2,\ldots$ be a sequence of random functions such that $g_i: \mathcal{Z} \longrightarrow \mathbb{R}$, and let $Z$ be a random quantity defined over $\mathcal{Z}$, independent of the random functions. Assume that $g(Z)$ is \revEJS{absolutely} continuous \revEJS{with respect to the Lebesgue measure}.
If, for every $z \in \mathcal{Z}$,
$$g_m(z)\xrightarrow[m \longrightarrow\infty]{\enskip \text{a.s.} \enskip} g(z),$$ 
then 
$$g_m(Z)\xrightarrow[m \longrightarrow\infty]{\enskip \mathcal{L} \enskip} g(Z).$$ 
\end{Lemma}
\begin{proof}
Fix $y \in \mathbb{R}$ and let $A_y=\{z \in \mathcal{Z}: g(z)\neq y \}$.
Notice that $\P(Z \in A_y)=1$. Moreover,
 the almost sure convergence of $g_m(z)$ implies
 its convergence in distribution. It follows that for every
$z \in A_y,$
\begin{align}
\label{eq:conv_dist_g}
 \lim_m\P(g_m(z)\leq  y) =  \P\left(g(z)\leq y\right). 
\end{align}
Now,
 using Equation \ref{eq:conv_dist_g} and Lebesgue's dominated convergence theorem, notice that 
\begin{align*}
    \lim_m \P(g_m(Z)<y)&=
    \lim_m\int_{\mathcal{Z}} \P(g_m(Z)<y|Z=z)d\P_Z(z)=
    \int_{\mathcal{Z}} \lim_m \P(g_m(Z)<y|Z=z)d\P_Z(z) \\
    &=\int_{A_z} \lim_m \P(g_m(z)<y)d\P_Z(z)=
    \int_{A_z}  \P(g(z)<y)d\P_Z(z)\\
    &=
    \int_{\mathcal{Z}}  \P(g(Z)<y|Z=z)d\P_Z(z)=
    \P(g(Z)<y),
\end{align*}
which concludes the proof.
\end{proof}
\vspace{4mm}
}

\begin{proof}[Proof of Theorem \ref{thm:pval_right_coverage}] 
Assumption \ref{assump:uniform_consistency} implies that, for every $D$,
\begin{align*}
    0 \leq |\hat p(D;\Theta_0)-  p(D;\Theta_0)| &=     |\sup_{\theta \in \Theta_0}  \hat p(D;\theta)-\sup_{\theta \in \Theta_0}  p(D;\theta)| \\
&\leq \sup_{\theta \in \Theta_0} | \hat p(D;\theta)-  p(D;\theta)|  \xrightarrow[B' \longrightarrow\infty]{\enskip \text{a.s.} \enskip}  0,
\end{align*}
and therefore $\hat p(D;\Theta_0)$ converges almost surely to $p(D;\Theta_0)$. 
It follows \revJMLR{from Lemma \ref{lemma:convergence_dist}} that
$\hat p(\D;\Theta_0)$ converges in distribution to $p(\D;\Theta_0)$.
Conclude that
\begin{align*}
    \P_{\D,\T'|\theta}(\hat p (\D;\Theta_0)\leq \alpha)=F_{\hat p (\D;\Theta_0)|\theta}(\alpha) \xrightarrow[B' \longrightarrow\infty]{} F_{ p (\D;\Theta_0)|\theta}(\alpha) =
    \P_{\D|\theta}(p (D;\Theta_0)\leq \alpha),
\end{align*}
where $F_Z$ denotes the cumulative distribution function of the random variable $Z$.
\end{proof}

\begin{proof}[Proof of Corollary \ref{cor:pval_right_coverage}]
Fix $\theta \in \Theta$. Because $F_{\theta}$ is continuous, the definition  of $p(\D;\theta)$
implies that its distribution is uniform under the null. Thus  $\P_{\D|\theta}\left(p (\D;\theta)\leq \alpha\right)=\alpha$. 
 Theorem \ref{thm:pval_right_coverage} therefore implies that
 \begin{align}
 \label{eq:equal_alpha}
 \P_{\D,\T'|\theta}(\hat p (\D;\theta)\leq \alpha) \xrightarrow[B' \longrightarrow\infty]{} \P_{\D|\theta}\left(p (\D;\theta)\leq \alpha\right)= \alpha.    
 \end{align}
 Now, for any $\theta \in \Theta_0$, uniformity of the p-value implies that
\begin{align*}
\P_{\D|\theta}(p (\D;\Theta_0)\leq \alpha)&= \P_{\D|\theta}\left(\sup_{\theta_0 \in \Theta_0}  p (\D;\theta_0)\leq \alpha\right) \leq \P_{\D|\theta}\left(  p (\D;\theta)\leq \alpha\right) \\
&=\alpha.
\end{align*}
Conclude from Theorem \ref{thm:pval_right_coverage}  that
\begin{align}
 \label{eq:smaller_alpha}
 \P_{\D,\T'|\theta}(\hat p (\D;\Theta_0)\leq \alpha) \xrightarrow[B' \longrightarrow\infty]{} \P_{\D|\theta}( p (\D;\Theta_0)\leq \alpha) \leq \alpha.
 \end{align}
The conclusion follows from putting together Equations \ref{eq:equal_alpha} and \ref{eq:smaller_alpha}.
\end{proof}

\begin{proof}[Proof of Theorem \ref{thm:pval_rate}]
\begin{align*}
|\hat p(D;\Theta_0)-  p(D;\Theta_0)| &=     |\sup_{\theta \in \Theta_0}  \hat p(D;\theta)-\sup_{\theta \in \Theta_0}  p(D;\theta)| \\
&\leq \sup_{\theta \in \Theta_0} | \hat p(D;\theta)-  p(D;\theta)| \\
&=O_P\left(\left(\frac{1}{B'}\right)^{r}\right),
\end{align*}
where the last line follows from Assumption \ref{assump:conv_reg_pval}
\end{proof}

\revEJS{
\begin{Lemma}
\label{lemma:bound_statistics}
Under Assumption \ref{assump:bounded_odds}, for every $\theta,\theta_0 \in \Theta$
$$\E^2_{\D|\theta,T}\left[ | \tau(\D;\theta_0)-\hat \tau_B(\D;\theta_0)|\right]\leq M^2
\int (\mathbb{O}(\x;\theta_0)-\hat{\mathbb{O}}(\x;\theta_0))^2 dG(\x).$$ 
\end{Lemma}
\begin{proof} 
For every $\theta \in \Theta$, 
\begin{align*}
    \E^2_{\D|\theta,T}[ |\tau(\D;\theta_0)-\hat\tau_B(\D;\theta_0)|] &= \left( \int |\tau(\D;\theta_0)-\hat\tau_B(\D;\theta_0)|\ dF(\x|\theta) \right)^2 \\
    &=
    \left( \int |\mathbb{O}(\x;\theta_0)-\hat{\mathbb{O}}(\x;\theta_0)| \ dF(\x|\theta)\right)^2\\
     &=  \left( \int |\mathbb{O}(\x;\theta_0)-\hat{\mathbb{O}}(\x;\theta_0)| \mathbb{O}(\x;\theta) dG(\x) \right)^2\\
     &\leq    \left( \int (\mathbb{O}(\x;\theta_0)-\hat{\mathbb{O}}(\x;\theta_0)^2 dG(\x) \right) \left(  \int  \mathbb{O}^2(\x;\theta) dG(\x) \right),
\end{align*}
where the last inequality follows from Cauchy-Schwarz.
Assumption \ref{assump:bounded_odds} implies that
$$ \int \mathbb{O}^2(\x;\theta)dG(\x)\leq M^2, $$
from which we conclude that
$$  \E^2_{\D|\theta,T}[ |\tau(\D;\theta_0)-\hat\tau_B(\D;\theta_0)|] \leq M^2
\int (\mathbb{O}(\x;\theta_0)-\hat{\mathbb{O}}(\x;\theta_0))^2 dG(\x).$$
\end{proof}
}

\revEJS{
\begin{Lemma}
\label{lemma:bound_different_tests}For fixed $c \in \mathbb{R}$, let 
$\phi_{\tau;\theta_0}(\D)=\I\left(\tau(\D;\theta_0)<c\right)$ and
$\phi_{\hat\tau_B;\theta_0}(\D)=\I\left(\hat \tau_B(\D;\theta_0)<c\right)$
be the testing procedures for testing  $H_{0,\theta_0}:\theta=\theta_0$  obtained using $\tau$ and $\hat \tau_B$.
Under Assumptions \ref{assump:bounded_odds}-\ref{assump:Lipschitz}, for every $0<\epsilon<1$,
$$ \P_{\mathcal{D}|\theta,T}(\phi_{\tau;\theta_0}(\D) \neq \phi_{\hat\tau_B;\theta_0}(\D)) \leq \frac{2MC_L\cdot \sqrt{ \int (\mathbb{O}(\x;\theta_0)-\hat{\mathbb{O}}(\x;\theta_0))^2 dG(\x) }}{\epsilon} + \epsilon.$$
\end{Lemma}
\begin{proof}[Proof of Lemma~\ref{lemma:bound_different_tests}] 
 It follows from Markov's  inequality  and Lemma \ref{lemma:bound_statistics} that with probability at least $1-\epsilon$, $\D$ is such that
    \begin{align}
    \label{eq:markov1}
    |\tau(\D;\theta_0)-\hat\tau(\D;\theta_0)|\leq \frac{M \cdot \sqrt{ \int (\mathbb{O}(\x;\theta_0)-\hat{\mathbb{O}}(\x;\theta_0))^2 dG(\x) }}{\epsilon}
    \end{align}
Now we upper bound $\P_{\mathcal{D}|\theta,T}(\phi_{\tau;\theta_0}(\D) \neq \phi_{\hat\tau;\theta_0}(\D))$. Define $A$ as the event that Eq.~\ref{eq:markov1} happens and  let $h(\theta_0):= \int (\mathbb{O}(\x;\theta_0)-\hat{\mathbb{O}}(\x;\theta_0))^2 dG(\x)$. Then:
\begin{align*}
  \P_{\mathcal{D}|\theta,T}(\phi_{\tau;\theta_0}(\D) \neq \phi_{\hat\tau;\theta_0}(\D))&\leq 
  \P_{\mathcal{D}|\theta,T}(\phi_{\tau;\theta_0}(\D) \neq \phi_{\hat\tau;\theta_0}(\D),A)+\P_{\theta}(A^c) \\
 & \leq
 \P_{\mathcal{D}|\theta,T}\left( \I\left(\tau(\D;\theta_0)<c\right) \neq \I\left(\hat\tau(\D;\theta_0)<c\right), A \right) +\epsilon \\ &
 \leq \P_{\mathcal{D}|\theta,T}\left( c-\frac{M\cdot \sqrt{h(\theta_0)}}{\epsilon}  <\tau(\D;\theta_0)<c+\frac{M\cdot \sqrt{h(\theta_0)}}{\epsilon} \right) +\epsilon
\end{align*}
Assumption \ref{assump:Lipschitz} then implies that
$$ \P_{\mathcal{D}|\theta,T}(\phi_{\tau;\theta_0}(\D) \neq \phi_{\hat\tau;\theta_0}(\D)) \leq \frac{K'\cdot \sqrt{h(\theta_0)}}{\epsilon} + \epsilon$$
where $K' = 2 M C_L$, which concludes the proof.
\end{proof}
}

\revEJS{
\begin{proof}[Proof of Theorem \ref{thm:bound_average}]
 Follows directly from Lemma \ref{lemma:bound_different_tests} and Jensen's inequality.
\end{proof} 
}

\revEJS{
\begin{Lemma}
\label{lemma:bound_expected_loss}
Under Assumptions \ref{assump:bounded_odds}-\ref{assump:bounded_measure},
there exists $C>0$ such that
$$  \E_\T \left[  L(\hat{\mathbb{O}},\mathbb{O}) \right]  \leq C B^{-\kappa / ((\kappa+d))}.$$
$$ $$
\end{Lemma}
\begin{proof}
Let $\hat p = \hat \P(Y=1|\x,\theta)$ and $p = \P(Y=1|\x,\theta)$ be the probabilistic classifier and true classification function, respectively, on the training sample $T$. Let $h(y)=\frac{y}{1-y}$ for $0<y<1$. A Taylor expansion of $h$ implies that  
$$\left(h(\hat p) - h(p) \right)^2  = \left(h(p) + R_1(\hat p) - h(p) \right)^2 = R_1(\hat p)^2,$$ 
where $ R_1(\hat p)=h'(\xi) (\hat p - p)$ for some $\xi$ between $p$ and $\hat p$. Also note that due to Assumption \ref{assump:bounded_odds},
$$ \exists a > 0 \textrm{ s.t. } p, \hat p > a,\ \forall x\in\mathcal{X}, \theta \in \Theta.$$
Thus, 
\begin{align*}
\E_{\T} &\left[\iint \left(h(\hat p) - h(p)  \right)^2 dG(\x) d\pi(\theta) \right] \\ &= \E_{\T} \left[ \iint \frac{ 1 }{ (1-\xi)^4} \left(\hat p - p\right)^2 dG(\x) d\pi(\theta) \right] \\  
& \leq \frac{ 1 }{ (1-a)^4} \E_{\T} \left[ \iint \left(\hat p - p\right)^2 dG(\x) d\pi(\theta) \right] \\
& = \frac{ 1 }{ (1-a)^4} \E_{\T} \left[ \int \left(\hat \P(Y=1|\x,\theta) - \P(Y=1|\x,\theta) \right)^2 h'(\x,\theta) dH(\x, \theta) \right] \\
& \leq \frac{ \gamma }{ (1-a)^4} \E_{\T} \left[ \int \left(\hat \P(Y=1|\x,\theta) - \P(Y=1|\x,\theta) \right)^2 dH(\x, \theta) \right] \\
& = O \left(B^{-\kappa / (\kappa+d) } \right).    
\end{align*}
\end{proof}
}

\revEJS{
\begin{proof}[Proof of Theorem \ref{thm:bound_different_tests_eps}]
It follows from Theorem \ref{thm:bound_average} that 
\begin{align*}
   \int  \P_{\mathcal{D},\T|\theta}(\phi_{\tau;\theta_0}(\D) \neq \phi_{\hat\tau_B;\theta_0}(\D)) d\pi(\theta_0)&= \E_\T \left[ \int  \P_{\mathcal{D}|\theta,T}(\phi_{\tau;\theta_0}(\D) \neq \phi_{\hat\tau_B;\theta_0}(\D)) d\pi(\theta_0) \right]\\
   &\leq \frac{2MC_L\cdot \E_\T \left[ \sqrt{ L(\hat{\mathbb{O}},\mathbb{O}) } \right]}{\epsilon} + \epsilon \\ 
   &\leq \frac{2MC_L\cdot\sqrt{ \E_\T \left[  L(\hat{\mathbb{O}},\mathbb{O}) \right]} }{\epsilon} + \epsilon,
\end{align*}
where the last step follows from Jensen's inequality. It follows from this and Lemma \ref{lemma:bound_expected_loss} that 
$$\int  \P_{\mathcal{D},\T|\theta}(\phi_{\tau;\theta_0}(\D) \neq \phi_{\hat\tau_B;\theta_0}(\D)) d\pi(\theta_0)\leq \frac{K B^{-\kappa / (2(\kappa+d))} }{\epsilon}  + \epsilon,$$
where $K = 2MC_L \sqrt{C}$. Notice that taking $\epsilon^* = \sqrt{K} B^{-\kappa / (4(\kappa+d))}$ optimizes the bound and gives the result.
\end{proof}
}

\revEJS{
\begin{proof}[Proof of Corollary \ref{coroll:power}]
The result follows from noticing that
$$\P_{\mathcal{D},\T|\theta}(\phi_{\hat\tau_B;\theta_0}(\D)=1)\geq \P_{\mathcal{D},\T|\theta}(\phi_{\tau;\theta_0}(\D)=1) - \P_{\mathcal{D},\T|\theta}(\phi_{\tau;\theta_0}(\D) \neq \phi_{\hat\tau_B;\theta_0}(\D)),$$
and therefore
\begin{align*}
   \int \P_{\mathcal{D},\T|\theta}(\phi_{\hat\tau_B;\theta_0}(\D)=1) d\theta_0  & \geq \int\P_{\mathcal{D},\T|\theta}(\phi_{\tau;\theta_0}(\D)=1) d\theta_0 - \int \P_{\mathcal{D},\T|\theta}(\phi_{\tau;\theta_0}(\D) \neq \phi_{\hat\tau_B;\theta_0}(\D))d\theta_0  \\
    & \geq \int \P_{\mathcal{D},\T|\theta}(\phi_{\tau;\theta_0}(\D)=1)d\theta_0  - K' B^{-\kappa / (4(\kappa+d))},
    \end{align*} 
where the last inequality follows from Theorem \ref{thm:bound_different_tests_eps}.
\end{proof}
}

\section{Loss Functions} \label{app:loss_function}

In this work, we use the cross-entropy loss to train probabilistic classifiers. Consider a sample point $\{\theta,\x,y\}$ generated according to Algorithm \ref{alg:joint_y}. Let
$p$ be a $\mbox{Bernoulli}(y)$ distribution, and
$q$ be a  $\mbox{Ber}\left(\widehat{\P}(Y=1 | \theta, \x)\right)=\mbox{Ber}\left(\frac{\hat{\mathbb{O}}(\x;\theta)}{1+\hat{\mathbb{O}}(\x;\theta)}\right)$ 
distribution. The {\em cross-entropy} between $p$ and $q$ is given by
\begin{align}
\label{eq::CE_loss}
\nonumber \mathcal{L}_{\mbox{CE}}(\hat{\mathbb{O}}; \{\theta,\x,y\}) &= - y \log\left(\frac{\hat{\mathbb{O}}(\x;\theta)}{1+\hat{\mathbb{O}}(\x;\theta)}\right) - (1-y) \log\left(\frac{1}{1+\hat{\mathbb{O}}(\x;\theta)}\right) \\
&= - y \log \left(\hat{\mathbb{O}}(\x;\theta) \right) + \log \left(1+\hat{\mathbb{O}}(\x;\theta) \right).
\end{align}
For every $\x$ and $\theta$, the expected cross-entropy $\E[L_{\mbox{CE}}(\hat{\mathbb{O}}; \{\theta,\x,y\})]$
is minimized by $\hat{\mathbb{O}}(\x;\theta)=\mathbb{O}(\x;\theta)$. If the probabilistic classifier attains the minimum of the cross-entropy loss, then the estimated \texttt{ACORE} statistic $\widehat \Lambda(\D; \Theta_0)$ will be equal to the likelihood ratio statistic in Equation~\ref{eq::LRT}, as shown in \cite{dalmasso2020ACORE}. Similarly, as stated in Proposition \ref{prop::consistency}, at the minimum, the estimated \texttt{BFF} statistic $\widehat\tau(\D; \Theta_0)$  is equal to the Bayes factor in Equation~\ref{eq::BF}.

\end{appendix}

\begin{acks}[Acknowledgments]
The authors would like to thank Mikael Kuusela, Rafael Stern and Larry Wasserman for helpful discussions. \addEJS{We are also indebted to Tommaso Dorigo, Jan Kieseler and Giles C. Strong for providing the muon energy data and the neural network architecture used for the studies described in Section \ref{sec:muons}.}
\end{acks}

\bibliographystyle{imsart-number} 
\bibliography{bibliography}


\end{document}